\documentclass[twoside,11pt]{article}

\usepackage{jmlr2e}
\usepackage{amsmath}
\usepackage{hyperref}
\usepackage{subcaption}

\newtheorem{thm}{Theorem}
\newtheorem{lem}[thm]{Lemma}
\newtheorem{prop}[thm]{Proposition}
\newtheorem{cor}[thm]{Corollary}
\newtheorem{defn}[thm]{Definition}
\newtheorem{exmp}{Example}

\newcommand{\Fig}[1]{Figure~\ref{#1}}
\newcommand{\Sec}[1]{Section~\ref{#1}}

\newcommand{\Eqn}[1]{Eq.~(\ref{#1})}

\newcommand{\Lem}[1]{Lemma~\ref{#1}}
\newcommand{\Thm}[1]{Theorem~\ref{#1}}
\newcommand{\Prop}[1]{Proposition~\ref{#1}}
\newcommand{\Cor}[1]{Corollary~\ref{#1}}
\newcommand{\App}[1]{Appendix~\ref{#1}}
\newcommand{\Def}[1]{Definition~\ref{#1}}
\newcommand{\Exmp}[1]{Example~\ref{#1}}
\renewcommand{\hat}{\widehat}
\renewcommand{\tilde}{\widetilde}
\newcommand{\bbar}{\overline}
\renewcommand{\>}{{\rightarrow}}
\renewcommand{\=}{\stackrel{\triangle}{=}}
\newcommand{\half}{{\frac {1}{2}}}
\newcommand{\argmax}{\operatorname{argmax}}
\newcommand{\argmin}{\operatorname{argmin}}
\newcommand{\affdim}{\operatorname{affdim}}
\newcommand{\conv}{\operatorname{conv}}
\renewcommand{\dim}{\operatorname{dim}}
\newcommand{\nullity}{\operatorname{nullity}}
\newcommand{\rank}{\operatorname{rank}}
\newcommand{\relint}{\operatorname{relint}}
\newcommand{\sign}{\operatorname{sign}}
\renewcommand{\span}{\operatorname{span}}
\newcommand{\R}{{\mathbb R}}
\newcommand{\Z}{{\mathbb Z}}
\newcommand{\N}{{\mathbb N}}
\renewcommand{\P}{{\mathbf P}}
\newcommand{\E}{{\mathbf E}}
\newcommand{\I}{{\mathbf I}}
\newcommand{\1}{{\mathbf 1}}
\newcommand{\0}{{\mathbf 0}}
\newcommand{\cA}{{\mathcal A}}
\newcommand{\A}{{\mathbf A}}

\newcommand{\B}{{\mathbf B}}
\newcommand{\C}{{\mathcal C}}

\newcommand{\F}{{\mathcal F}}
\newcommand{\G}{{\mathcal G}}
\renewcommand{\H}{{\mathcal H}}

\renewcommand{\L}{{\mathbf L}}
\newcommand{\M}{{\mathbf M}}
\newcommand{\cN}{{\mathcal N}}
\newcommand{\cP}{{\mathcal P}}
\newcommand{\Q}{{\mathcal Q}}
\newcommand{\cR}{{\mathcal R}}
\renewcommand{\S}{{\mathcal S}}

\newcommand{\V}{{\mathcal V}}
\newcommand{\X}{{\mathcal X}}
\newcommand{\Y}{{\mathcal Y}}

\renewcommand{\a}{{\mathbf a}}
\renewcommand{\b}{{\mathbf b}}
\newcommand{\e}{{\mathbf e}}
\newcommand{\f}{{\mathbf f}}

\newcommand{\p}{{\mathbf p}}
\newcommand{\q}{{\mathbf q}}
\renewcommand{\r}{{\mathbf r}}

\renewcommand{\t}{{\mathbf t}}
\renewcommand{\u}{{\mathbf u}}
\renewcommand{\v}{{\mathbf v}}
\newcommand{\w}{{\mathbf w}}

\newcommand{\y}{{\mathbf y}}
\newcommand{\z}{{\mathbf z}}

\newcommand{\bell}{{\boldsymbol \ell}}
\newcommand{\balpha}{{\boldsymbol \alpha}}

\newcommand{\bsigma}{{\boldsymbol \sigma}}

\newcommand{\bpsi}{{\boldsymbol \psi}}

\newcommand{\er}{\textup{\textrm{er}}}
\newcommand{\pred}{\textup{\textrm{pred}}}
\newcommand{\CCdim}{\operatorname{CCdim}}
\newcommand{\reg}{\textup{\textrm{regret}}}

\newcommand{\Null}{\textup{\textrm{null}}}

\newcommand{\aff}{\textup{\textrm{aff}}}

\newcommand{\cl}{\textup{\textrm{cl}}}
\newcommand{\zo}{\textup{\textrm{0-1}}}
\newcommand{\ord}{\textup{\textrm{ord}}}
\newcommand{\Ham}{\textup{\textrm{Ham}}}
\newcommand{\abstain}{\textup{\textrm{(?)}}}
\newcommand{\CS}{\textup{\textrm{CS}}}
\newcommand{\abs}{\textup{\textrm{abs}}}
\newcommand{\MAP}{\textup{\textrm{MAP}}}

\newcommand{\NDCG}{\textup{\textrm{NDCG}}}
\newcommand{\pair}{\textup{\textrm{PD}}}

\jmlrheading{}{}{}{}{}{Harish G. Ramaswamy and Shivani Agarwal}

\ShortHeadings{Convex Calibration Dimension for Multiclass Loss Matrices}{Ramaswamy and Agarwal}
\firstpageno{1}

\begin{document}

\title{Convex Calibration Dimension for \\Multiclass Loss Matrices  }

\author{\name Harish G.\ Ramaswamy \email harish\_gurup@csa.iisc.ernet.in \\[-18pt]
       \AND
       \name Shivani Agarwal \email shivani@csa.iisc.ernet.in \\
       \addr Department of Computer Science and Automation\\
       Indian Institute of Science\\
       Bangalore 560012, India}

\editor{}

\maketitle


\begin{abstract}
We study consistency properties of surrogate loss functions for general multiclass learning problems, defined by a general multiclass loss matrix. We extend the notion of classification calibration, which has been studied for binary and multiclass 0-1 classification problems (and for certain other specific learning problems), to the general multiclass setting, and derive necessary and sufficient conditions for a surrogate loss to be calibrated with respect to a loss matrix in this setting. We then introduce the notion of \emph{convex calibration dimension} of a multiclass loss matrix, which measures the smallest `size' of a prediction space in which it is possible to design a convex surrogate that is calibrated with respect to the loss matrix. We derive both upper and lower bounds on this quantity, and use these results to analyze various loss matrices. In particular, we apply our framework to study various subset ranking losses, and
use the convex calibration dimension as a tool to show both the existence and non-existence of various types of convex calibrated surrogates for these losses. Our results strengthen recent results of Duchi et al.\ (2010) and Calauz\`{e}nes et al.\ (2012) on the non-existence of certain types of convex calibrated surrogates 
in subset ranking. We anticipate the convex calibration dimension may prove to be a useful tool in the study and design of surrogate losses for general multiclass learning problems.
\end{abstract}

\begin{keywords}
Statistical consistency, multiclass loss, loss matrix, surrogate loss, convex surrogates, calibrated surrogates, classification calibration, subset ranking.
\end{keywords}

\section{Introduction}
\label{sec:intro}

There has been significant interest and progress in recent years in understanding consistency properties of surrogate risk minimization algorithms for various learning problems, such as binary classification, multiclass 0-1 classification, and various forms of ranking and multi-label prediction problems \citep{LugosiVa04,Jiang04,Zhang04a,Steinwart05,Bartlett+06,Zhang04b,TewariBa07,Steinwart07,CossockZh08,Xia+08,Duchi+10,Ravikumar+11,Buffoni+11,GaoZh11,Kotlowski+11}. Any such problem that involves a finite number of class labels and predictions can be viewed as an instance of a general multiclass learning problem, whose structure is defined by a suitable loss matrix. 
While the above studies have enabled an understanding of learning problems corresponding to certain forms of loss matrices, a framework for analyzing consistency properties for a general multiclass problem, defined by a general loss matrix, has remained elusive.

In this paper, we develop a unified framework for studying consistency properties of surrogate losses for such general multiclass learning problems, defined by a general multiclass loss matrix.
For algorithms minimizing a surrogate loss, the question of consistency with respect to the target loss matrix reduces to the question of \emph{calibration} of the surrogate loss with respect to the target loss.\footnote{Assuming the surrogate risk minimization procedure is itself consistent (with respect to the surrogate loss); in most cases, this can be achieved by minimizing the surrogate risk over a function class that approaches a universal function class as the training sample size increases, e.g.\ see \cite{Bartlett+06}.} We start by giving both necessary and sufficient conditions for a surrogate loss function to be calibrated with respect to any given target loss matrix. These conditions generalize previous conditions for the multiclass 0-1 loss studied for example by \cite{TewariBa07}. We then introduce the notion of  \emph{convex calibration dimension} of a loss matrix, a fundamental quantity that measures the smallest `size' of a prediction space in which it is possible to design a convex surrogate that is calibrated with respect to the given loss matrix. This quantity can be viewed as representing one measure of the intrinsic `difficulty' of the loss, and has a non-trivial behavior in the sense that one can give examples of loss matrices defined on the same number of class labels that have very different values of the convex calibration dimension, ranging from one (in which case one can achieve consistency by learning a single real-valued function) to practically the number of classes (in which case one must learn as many real-valued functions as the number of classes). 
We give upper and lower bounds on this quantity in terms of various algebraic and geometric properties of the loss matrix, and apply these results to analyze various loss matrices. 

As concrete applications of our framework, we use the convex calibration dimension as a tool to study various loss matrices that arise in subset ranking problems, including the normalized discounted cumulative gain (NDCG), pairwise disagreement (PD), and mean average precision (MAP) losses.
A popular practice in subset ranking, where one needs to rank a set of $r$ documents by relevance to a query, has been to learn $r$ real-valued scoring functions by minimizing a convex surrogate loss in $r$ dimensions, and to then sort the $r$ documents based on these scores. As discussed recently by \cite{Duchi+10} and \cite{Calauzenes+12}, such an approach cannot be consistent for the PD and MAP losses, since these losses do not admit convex calibrated surrogates in $r$ dimensions that can be used together with the sorting operation. We obtain a stronger result; in particular, we show that the convex calibration dimension of these losses is lower bounded by a quadratic function of $r$, which means that if minimizing a convex surrogate loss, one necessarily needs to learn $\Omega(r^2)$ real-valued functions to achieve consistency for these losses. 

\subsection{Related Work}
\label{subsec:related-work}

There has been much work in recent years on consistency and calibration of surrogate losses for various learning problems. We give a brief overview of this body of work here.

Initial work on consistency of surrogate risk minimization algorithms focused largely on binary classification. For example, \cite{Steinwart05} showed the consistency of support vector machines with universal kernels for the problem of binary classification; \cite{Jiang04} and \cite{LugosiVa04} showed similar results for boosting methods.
\cite{Bartlett+06} and \cite{Zhang04a} studied the calibration of margin-based surrogates for binary classification. In particular, in their seminal work, \cite{Bartlett+06} established that the property of `classification calibration' of a surrogate loss is equivalent to its minimization yielding 0-1 consistency, 
and gave a simple necessary and sufficient condition for convex margin-based surrogates to be calibrated w.r.t.\ the binary 0-1 loss. More recently, \cite{ReidWi10} analyzed the calibration of a general family of surrogates termed proper composite surrogates for binary classification. Variants of standard 0-1 binary classification have also been studied; for example, \cite{YuanWe10} studied consistency for the problem of binary classification with a reject option, and \cite{Scott12} studied calibrated surrogates for cost-sensitive binary classification. 

Over the years, there has been significant interest in extending the understanding of consistency and calibrated surrogates to various multiclass learning problems. Early work in this direction, pioneered by \cite{Zhang04b} and \cite{TewariBa07}, considered mainly the multiclass 0-1 classification problem. This work generalized the framework of  \cite{Bartlett+06} to the multiclass 0-1 setting and used these results to study calibration of various surrogates proposed for multiclass 0-1 classification, such as the surrogates of \cite{WestonWa99}, \cite{CrammerSi01a}, and \cite{Lee+04}. In particular, while the multiclass surrogate of \cite{Lee+04} was shown to calibrated for multiclass 0-1 classification, it was shown that several other widely used multiclass surrogates are in fact not calibrated for multiclass 0-1 classification.

More recently, there has been much work on studying consistency and calibration for various other learning problems that also involve finite label and prediction spaces. For example, \cite{GaoZh11} studied consistency and calibration for multi-label prediction with the Hamming loss. Another prominent class of learning problems for which consistency and calibration have been studied recently is that of subset ranking, where instances contain queries together with sets of documents, and the goal is to learn a prediction model that given such an instance ranks the documents by relevance to the query. Various subset ranking losses have been investigated in recent years. \cite{CossockZh08} studied subset ranking with the discounted cumulative gain (DCG) ranking loss, and gave a simple surrogate calibrated w.r.t.\ this loss; \cite{Ravikumar+11} further studied subset ranking with the normalized DCG (NDCG) loss.
\cite{Xia+08} considered the 0-1 loss applied to permutations. \cite{Duchi+10} focused on subset ranking with the pairwise disagreement (PD) loss, and showed that several popular convex score-based surrogates used for this problem are in fact not calibrated w.r.t.\ this loss; they also conjectured that such surrogates may not exist. 
\cite{Calauzenes+12} showed conclusively that there do not exist any convex score-based surrogates that are calibrated w.r.t.\ the PD loss, or w.r.t.\ the mean average precision (MAP) or expected reciprocal rank (ERR) losses.
Finally, in a more general study of subset ranking losses, \cite{Buffoni+11} introduced the notion of `standardization' for subset ranking losses, and gave a way to construct convex calibrated score-based surrogates for subset ranking losses that can be `standardized'; they showed that while the DCG and NDCG losses can be standardized, the MAP and ERR losses cannot be standardized. 

We also point out that in a related but different context, consistency of ranking has also been studied in the instance ranking setting \citep{ClemenconVa07,Clemencon+08,Kotlowski+11,Agarwal14}. 

Finally, \cite{Steinwart07} considered consistency and calibration in a very general setting.
More recently, \cite{Pires+13} used Steinwart's techniques to obtain surrogate regret bounds for certain surrogates w.r.t.\ general multiclass losses, and \cite{Ramaswamy+13} showed how to design explicit convex calibrated surrogates for any low-rank loss matrix. 

\subsection{Contributions of this Paper}
\label{subsec:contributions}

As noted above, we develop a unified framework for studying consistency and calibration for general multiclass (finite-output) learning problems, described by a general loss matrix. We give both necessary conditions and sufficient conditions for a surrogate loss to be calibrated w.r.t.\ a given multiclass loss matrix, and introduce the notion of \emph{convex calibration dimension} of a loss matrix, which measures the smallest `size' of a prediction space in which it is possible to design a convex surrogate that is calibrated with respect to the loss matrix. We derive both upper and lower bounds on this quantity in terms of certain algebraic and geometric properties of the loss matrix, and apply these results to study various subset ranking losses. In particular, we obtain stronger results on the non-existence of convex calibrated surrogates for certain types of subset ranking losses than previous results in the literature (and also positive results on the existence of convex calibrated surrogates for these losses in higher dimensions). The following is a summary of the main differences from the conference version of this paper \citep{RamaswamyAg12}:
\begin{itemize}
\item
Enhanced definition of positive normal sets of a surrogate loss at a sequence of points (\Def{def:pos-normals}; this is required for proofs of stronger versions of our earlier results).
\vspace{-4pt}
\item
Stronger necessary condition for calibration (\Thm{thm:condition-necessary-sequence}).
\vspace{-4pt}
\item
Stronger versions of upper and lower bounds on the convex calibration dimension, with full proofs (Theorems~\ref{thm:ccdim-upper-bound}, \ref{thm:ccdim-lower-bound}).
\vspace{-4pt}
\item
Conditions under which the upper and lower bounds are tight (\Sec{subsec:tight}).
\vspace{-4pt}
\item
Application to a more general setting of the PD loss (\Sec{subsec:PD}).
\vspace{-4pt}
\item
Additional applications to the NDCG and MAP losses (Sections~\ref{subsec:NDCG}, \ref{subsec:MAP}).
\vspace{-4pt}
\item
Additional examples and illustrations throughout (Examples~\ref{exmp:surrogate-CS}, \ref{exmp:calibration-CS-0-1}, \ref{exmp:calibration-CS-abstain-ord}, \ref{exmp:pos-normals-eps-surrogate}; Figures~\ref{fig:pos-normals}, \ref{fig:necess-cond-illus}).\\
\vspace{-16pt}
\item
Minor improvements and changes in emphasis in notation and terminology.
\end{itemize}

\subsection{Organization}
\label{subsec:organization}

We start in \Sec{sec:prelims} with some preliminaries and examples that will be used as running examples to illustrate concepts throughout the paper, and formalize the notion of calibration with respect to a general multiclass loss matrix. In \Sec{sec:conditions}, we derive both necessary conditions and sufficient conditions for calibration with respect to general loss matrices; these are both of independent interest and useful in our later results. \Sec{sec:ccdim} introduces the notion of convex calibration dimension of a loss matrix
and derives both upper and lower bounds on this quantity. 
In \Sec{sec:ranking}, we apply our results to study the convex calibration dimension of various subset ranking losses.
We conclude with a brief discussion in \Sec{sec:concl}.
Shorter proofs are included in the main text; all longer proofs are collected in \Sec{sec:proofs} so as to maintain easy readability of the main text. The only exception to this is proofs of \Lem{lem:pred-pred'} and \Thm{thm:class-calibration-consistency}, which closely follow earlier proofs of  \cite{TewariBa07} and are included for completeness in \App{app:proofs}.
Some calculations are given in Appendices~\ref{app:calculations-trigger-prob} and \ref{app:calculations-pos-normals}.

\section{Preliminaries, Examples, and Background}
\label{sec:prelims}

In this section we set up basic notation (\Sec{subsec:notation}), give background on multiclass loss matrices and risks (\Sec{subsec:multiclass-losses}) and on multiclass surrogates and calibration (\Sec{subsec:multiclass-surrogates}), and then define certain properties associated with multiclass losses and surrogates that will be useful in our study (\Sec{subsec:trigger}).

\subsection{Notation}
\label{subsec:notation} 

Throughout the paper, we denote $\R = (-\infty,\infty)$, $\R_+ = [0,\infty)$, $\bbar{\R} = [-\infty,\infty]$, $\bbar{\R}_+ = [0,\infty]$. Similarly, $\Z$ and $\Z_+$ denote the sets of all integers and non-negative integers, respectively. For $n\in\Z_+$, we denote $[n] = \{1,\ldots,n\}$. 
For a predicate $\phi$, we denote by $\1(\phi)$ the indicator of $\phi$, which takes the value 1 if $\phi$ is true and 0 otherwise. For $z\in\R$, we denote $z_+ = \max(0,z)$.
For a vector $\v\in\R^n$, we denote $\|\v\|_0 = \sum_{i=1}^n \1(v_i \neq 0)$ and $\|\v\|_1 = \sum_{i=1}^n |v_i|$.
For a set $\cA\subseteq\R^n$, we denote by $\relint(\cA)$ the relative interior of $\cA$, by $\cl(\cA)$ the closure of $\cA$, by $\span(\cA)$ the linear span (or linear hull) of $\cA$, by $\aff(\cA)$ the affine hull of $\cA$, and by $\conv(\cA)$ the convex hull of $\cA$.
For a vector space $\V$, we denote by $\dim(\V)$ the dimension of $\V$.
For a matrix $\M\in\R^{m\times n}$, we denote by $\rank(\M)$ the rank of $\M$, by $\affdim(\M)$ the affine dimension of the set of columns of $\M$ (i.e.\ the dimension of the subspace parallel to the affine hull of the columns of $\M$), by $\Null(\M)$ the null space of $\M$, and by $\nullity(\M)$ the nullity of $\M$ (i.e.\ the dimension of the null space of $\M$). 
We denote by $\Delta_n$ the probability simplex in $\R^n$: $\Delta_n=\{\p\in\R_+^n: \sum_{i=1}^n p_i=1\}$.
Finally, we denote by $\Pi_n$ the set of all permutations of $[n]$, i.e.\ the set of all bijective mappings $\sigma:[n]\>[n]$; for a permutation $\sigma\in\Pi_n$ and element $i\in[n]$, $\sigma(i)$ therefore represents the position of element $i$ under $\sigma$. 

\subsection{Multiclass Losses and Risks} 
\label{subsec:multiclass-losses}

The general multiclass learning problem we consider can be described as follows:  There is a finite set of \emph{class labels} $\Y$ and a finite set of possible \emph{predictions} $\hat{\Y}$, which we take without loss of generality to be $\Y = [n]$ and $\hat{\Y} = [k]$ for some $n,k\in\Z_+$. We are given training examples $(X_1,Y_1),\ldots,(X_m,Y_m)$ drawn i.i.d.\ from a distribution $D$ on $\X\times\Y$, where $\X$ is an instance space, and the goal is to learn from these examples a prediction model $h:\X\>\hat{\Y}$ which given a new instance $x\in\X$, makes a prediction $\hat{y} = h(x) \in \hat{\Y}$. In many common learning problems, the label and prediction spaces are the same, i.e.\  $\hat{\Y}=\Y$, but in general, these could be different (e.g.\ when there is an `abstain' option available to a classifier, in which case $k = n+1$).

The performance of a prediction model is measured via a \emph{loss function} $\ell:\Y\times\hat{\Y}\>\R_+$, or equivalently, by a \emph{loss matrix} $\L\in\R_+^{n\times k}$, with $(y,t)$-th element given by $\ell_{yt} = \ell(y,t)$; here $\ell_{yt} = \ell(y,t)$ defines the penalty incurred on predicting $t\in[k]$ when the true label is $y\in[n]$. We will use the notions of loss matrix and loss function interchangeably.
Some examples of common multiclass loss functions and corresponding loss matrices are given below:


\begin{exmp}[0-1 loss]
\label{exmp:loss-0-1}
Here $\Y=\hat{\Y}=[n]$, and the loss incurred is $1$ if the predicted label $t$ is different from the actual class label $y$, and $0$ otherwise:
\[
\ell^\zo(y,t) 
	~ = ~
	\1\big( t\neq y \big)
	~~~~~~~~\forall y,t\in[n]
    \,.
\] 
The loss matrix $\L^{\zo}$ for $n=3$ is shown in \Fig{fig:loss-matrices}(a). 
This is one of the most commonly used multiclass losses, and is suitable when all prediction errors are considered equal.
\end{exmp}

\begin{exmp}[Ordinal regression loss]
\label{exmp:loss-ord}
Here $\Y=\hat{\Y}=[n]$, and predictions $t$ farther away from the actual class label $y$ are penalized more heavily, e.g.\ using absolute distance:
\[
\ell^\ord(y,t)
	~ = ~
	|t-y|
	~~~~~~~~\forall y,t\in[n]
    \,.
\]
The loss matrix $\L^{\ord}$ for $n=3$ is shown in \Fig{fig:loss-matrices}(b).
This loss is often used when the class labels satisfy a natural ordinal property, for example in evaluating recommender systems that predict the number of stars (say out of 5) assigned to a product by user.
\end{exmp}

\begin{exmp}[Hamming loss]
\label{exmp:loss-hamming}
Here $\Y=\hat{\Y}=[2^r]$ for some $r\in\Z_+$, and the loss incurred on predicting $t$ when the actual class label is $y$ is the number of bit-positions in which the $r$-bit binary representations of $t-1$ and $y-1$ differ:
\[
\ell^\Ham(y,t) 
	~ = ~
	\sum_{i=1}^r \1\big( (t-1)_i \neq (y-1)_i \big)
	~~~~~~~~\forall y,t\in[2^r]
    \,,
\]
where for each $z\in\{0,\ldots,2^r-1\}$, $z_i\in\{0,1\}$ denotes the $i$-th bit in the $r$-bit binary representation of $z$. The loss matrix $\L^{\Ham}$ for $r=2$ is shown in \Fig{fig:loss-matrices}(c).
This loss is frequently used in sequence learning applications, where each element in $\Y = \hat{\Y}$ is a binary sequence of length $r$, and the loss in predicting a sequence $\t\in\{0,1\}^r$ when the true label sequence is $\y\in\{0,1\}^r$ is simply the Hamming distance between the two sequences. 
\end{exmp}

\begin{exmp}[`Abstain' loss]
\label{exmp:loss-abstain}
Here $\Y=[n]$ and $\hat{\Y}=[n+1]$, where $t=n+1$ denotes a prediction of `abstain' (or `reject'). One possible loss function in this setting assigns a loss of $1$ to incorrect predictions in $[n]$, $0$ to correct predictions, and $\half$ for abstaining:
\[
\ell^\abstain(y,t) 
	~ = ~ 
	\1\big( t\in[n] \big) \cdot \1\big( t\neq y \big)  + \frac{1}{2} \cdot \1\big( t=n+1 \big)
	~~~~~~~~\forall y \in [n], t\in[n+1]
    \,.
\]
The loss matrix $\L^{\abstain}$ for $n=3$ is shown in \Fig{fig:loss-matrices}(d).
This type of loss is suitable in applications where making an erroneous prediction is more costly than simply abstaining. For example, in medical diagnosis applications, when uncertain about the correct prediction, it may be better to abstain and request human intervention rather than make a misdiagnosis.
\end{exmp}

\begin{figure}[t]
\[
{\underset{\textrm{\rule{0pt}{11pt}\normalsize{(a)}}}{\left[ \begin{array}{ccc}
    0 & 1 & 1 \\
    1 & 0 & 1 \\
    1 & 1 & 0
\end{array} \right]}}
\hspace{0.5cm}
{\underset{\textrm{\rule{0pt}{11pt}\normalsize{(b)}}}{\left[ \begin{array}{ccc}
    0 & 1 & 2 \\
    1 & 0 & 1 \\
    2 & 1 & 0
\end{array} \right]}}
\hspace{0.5cm}
{\underset{\textrm{\normalsize{(c)}}}{\left[ \begin{array}{cccc}
 	0 & 1 & 1 & 2 \\
    1 & 0 & 2 & 1 \\
    1 & 2 & 0 & 1 \\
    2 & 1 & 1 & 0
\end{array} \right]}}
\hspace{0.5cm}
{\underset{\textrm{\rule{0pt}{11pt}\normalsize{(d)}}}{\left[ \begin{array}{cccc}
 	0 & 1 & 1 & \half \\
    1 & 0 & 1 & \half \\
    1 & 1 & 0 & \half \\
\end{array} \right]}}
\]
\vspace{-12pt}
\caption{Loss matrices corresponding to Examples 1-4: (a) $\L^\zo$ for $n=3$; (b) $\L^\ord$ for $n=3$; (c) $\L^\Ham$ for $r=2$ ($n=4$); (d) $\L^\abstain$ for $n=3$.}
\label{fig:loss-matrices}
\end{figure}

As noted above, given examples $(X_1,Y_1),\ldots,(X_m,Y_m)$ drawn i.i.d. from a distribution $D$ on $\X\times[n]$, the goal is to learn a prediction model $h:\X\>[k]$. More specifically, given a target loss matrix $\L\in\R_+^{n\times k}$ with $(y,t)$-th element $\ell_{yt}$, the goal is to learn a model $h:\X\>[k]$ with small expected loss on a new example drawn randomly from $D$, which we will refer to as the \emph{$\L$-risk} or \emph{$\L$-error} of $h$:
\begin{equation}
\er_D^\L[h]
    ~ \= ~
    \E_{(X,Y)\sim D} \big[ \ell_{Y,h(X)} \big]
    \,.
\label{eqn:ell-risk}
\end{equation}
Clearly, denoting the $t$-th column of $\L$ as $\bell_t = (\ell_{1t},\ldots,\ell_{nt})^\top \in \R_+^n$, and the class probability vector at an instance $x\in\X$ under $D$ as $\p(x) = (p_1(x),\ldots,p_n(x))^\top \in \Delta_n$, where $p_y(x) = \P(Y=y \mid X=x)$ under $D$, the $\L$-risk of $h$ can be written as  
\begin{equation}    
\er_D^\L[h]
    ~ = ~
    \E_X \bigg[ \sum_{y=1}^n p_y(X) \, \ell_{y,h(X)} \bigg]
    ~ = ~
    \E_X \Big[ \p(X)^\top \bell_{h(X)} \Big]
    \,.
\end{equation}
The \emph{optimal $\L$-risk} or  \emph{optimal $\L$-error} for a distribution $D$ is then simply the smallest $\L$-risk or $\L$-error that can be achieved by any model $h$:
\begin{equation}
\er^{\L,*}_D
    ~ \= ~
    \inf_{h:\X\>[k]} \er^{\L}_D[h]
    ~ = ~
    \inf_{h:\X\>[k]} \E_X \Big[ \p(X)^\top \bell_{h(X)} \Big]
    ~ = ~ \E_X \bigg[ \min_{t\in[k]} \, \p(X)^\top \bell_t \bigg]
    \,.
\label{eqn:ell-risk-opt}
\end{equation}

Ideally, one would like to minimize (approximately) the $\L$-risk, e.g.\ by selecting a model that minimizes the average $\L$-loss on the training examples among some suitable class of models. However, minimizing the discrete $\L$-risk directly is typically computationally difficult. Consequently, one usually minimizes a (convex) \emph{surrogate} risk instead. 

\subsection{Multiclass Surrogates and Calibration} 
\label{subsec:multiclass-surrogates}

Let $d\in\Z_+$, and let $\C\subseteq\R^d$ be a convex set. A \emph{surrogate loss function} $\psi:\Y\times\C\>\R_+$ acting on the \emph{surrogate prediction space} $\C$ assigns a penalty $\psi(y,\u)$ on making a surrogate prediction $\u\in\C$ when the true label is $y\in[n]$.
The \emph{$\psi$-risk} or \emph{$\psi$-error} of a surrogate prediction model $\f:\X\>\C$ w.r.t.\ a distribution $D$ on $\X\times[n]$ is then defined as 
\begin{equation}
\er^\psi_D[\f]
    ~ \= ~ 
    \E_{(X,Y)\sim D} \Big[ \psi\big( Y,\f(X) \big) \Big]
    \,.
\label{eqn:psi-risk}
\end{equation}
The surrogate $\psi$ can be represented via $n$ real-valued functions $\psi_y:\C\>\R_+$ for $y\in[n]$, defined as
\(
\psi_y(\u) = \psi(y,\u)
	\,;
\)
equivalently, we can also represent the surrogate $\psi$ as a vector-valued function $\bpsi:\C\>\R_+^n$, defined as
\(
\bpsi(\u) = \big( \psi_1(\u),\ldots,\psi_n(\u) \big)^\top 
	\,.
\)
Clearly, the $\psi$-risk of $\f$ can then be written as
\begin{equation}
\er^\psi_D[\f]
    ~ = ~
    \E_X \bigg[ \sum_{y=1}^n p_y(X) \, \psi_y(\f(X)) \bigg]
    ~ = ~
    \E_X \Big[ \p(X)^\top \bpsi(\f(X)) \Big]
    \,.
\end{equation}
The \emph{optimal $\psi$-risk} or  \emph{optimal $\psi$-error} for a distribution $D$ is then simply the smallest $\psi$-risk or $\psi$-error that can be achieved by any model $\f$:
\begin{equation}
\er^{\psi,*}_D
    ~ \= ~
    \inf_{\f:\X\>\C} \er^{\psi}_D[\f]
    ~ = ~
    \inf_{\f:\X\>\C} \E_X \Big[ \p(X)^\top \bpsi(\f(X)) \Big]
    ~ = ~ 
    \E_X \bigg[ \inf_{\u\in\C} \p(X)^\top \bpsi(\u) \bigg]
    \,.
\label{eqn:psi-risk-opt}
\end{equation}
We will find it convenient to define the sets 
\begin{eqnarray}
\cR_\psi
    & \= &
    \bpsi(\C) ~ \subseteq ~ \R_+^n
\\
\S_\psi
    & \= &
    \conv(\cR_\psi) ~ \subseteq ~ \R_+^n
	\,.
\end{eqnarray}
Clearly, the optimal $\psi$-risk can then also be written as 
\begin{equation}
\er^{\psi,*}_D
    ~ = ~
	\E_X \bigg[ \inf_{\z\in\cR_\psi} \p(X)^\top \z \bigg]
    ~ = ~
	\E_X \bigg[ \inf_{\z\in\S_\psi} \p(X)^\top \z \bigg]
	\,.
\end{equation}

\begin{exmp}[Crammer-Singer surrogate]
\label{exmp:surrogate-CS}
The Crammer-Singer surrogate was proposed as a hinge-like surrogate loss for 0-1 multiclass classification \citep{CrammerSi01a}. For $\Y=[n]$, the Crammer-Singer surrogate $\bpsi^\CS$ acts on the surrogate prediction space $\C=\R^n$ and is defined as follows:
\[
\psi^\CS_y(\u)
	~ = ~
	\max_{y'\in [n], y'\neq y }  \big( 1 - (u_y - u_{y'}) \big)_+
	~~~~~~~~\forall y\in[n], \u\in\R^n
	\,.
\]
\end{exmp}

A surrogate $\bpsi$ is convex if $\psi_y$ is convex $\forall y\in[n]$. As an example, the Crammer-Singer surrogate defined above is clearly convex. Given training examples $(X_1,Y_1),\ldots,(X_m,Y_m)$ drawn i.i.d.\ from a distribution $D$ on $\X\times[n]$, a (convex) surrogate risk minimization algorithm using a (convex) surrogate loss $\bpsi:\C\>\R_+^n$ learns a surrogate prediction model by minimizing (approximately, based on the training sample) the $\psi$-risk; the learned model $\f:\X\>\C$ is then used to make predictions in the original space $[k]$ via some transformation $\pred:\C\>[k]$. This yields a prediction model $h:\X\>[k]$ for the original multiclass problem given by $h = (\pred \circ \f)$: the prediction on a new instance $x\in\X$ is given by $\pred(\f(x))$, and the $\L$-risk incurred is $\er^\L_D[\pred \circ \f]$. As an example, several surrogate risk minimizing algorithms for multiclass classification with respect to 0-1 loss (including that based on the Crammer-Singer surrogate) use a surrogate space $\C=\R^n$, learn a function of the form $\f:\X\>\R^n$, and predict according to $\pred(\f(x)) = \argmax_{t\in[n]}f_t(x)$.

Under suitable conditions, surrogate risk minimization algorithms that approximately minimize the $\psi$-risk based on a training sample are known to be consistent with respect to the $\psi$-risk, i.e.\ to converge (in probability) to the optimal $\psi$-risk as the number of training examples $m$ increases.
This raises the natural question of whether, for a given loss matrix $\L$, there are surrogate losses $\psi$ for which consistency with respect to the $\psi$-risk also guarantees consistency with respect to the $\L$-risk, i.e.\ guarantees convergence (in probability) to the optimal $\L$-risk (defined in \Eqn{eqn:ell-risk-opt}).
As we shall see below, this amounts to the question of \emph{calibration} of surrogate losses $\psi$ w.r.t.\ a given target loss matrix $\L$, and has been studied in detail for the 0-1 loss and for square losses of the form $\ell(y,t) = a_y \1(t\neq y)$, which can be analyzed similarly to the 0-1 loss \citep{Zhang04b,TewariBa07}. In this paper, we consider this question for general multiclass loss matrices $\L\in\R_+^{n\times k}$, including rectangular loss matrices with $k\neq n$. 
The only assumption we make on $\L$ is that for each $t\in[k]$, $\exists \p\in\Delta_n$ such that $\argmin_{t'\in[k]} \p^\top\bell_{t'} = \{t\}$ (otherwise the element $t \in[k]$ never needs to be predicted and can simply be ignored).

We will need the following definitions and basic results, generalizing those of \cite{Zhang04b}, \cite{Bartlett+06}, and \cite{TewariBa07}. The notion of calibration will be central to our study; as \Thm{thm:class-calibration-consistency} below shows, calibration of a surrogate loss $\psi$ w.r.t. $\L$ corresponds to the property that consistency w.r.t.\ $\psi$-risk implies consistency w.r.t.\ $\L$-risk.  
Proofs of \Lem{lem:pred-pred'} and \Thm{thm:class-calibration-consistency} can be found in \App{app:proofs}.

\begin{defn}[$(\L,\cP)$-calibration] 
\label{def:calibration}
Let $\L\in\R_+^{n\times k}$ and $\cP\subseteq \Delta_n$. 
A surrogate loss function $\bpsi:\C\>\R_+^n$ is said to be \emph{$(\L,\cP)$-calibrated} if there exists a function $\pred:\C\>[k]$ such that
\[
\forall \p\in\cP:
    ~~~~\inf_{\u\in\C:\pred(\u)\notin\argmin_{t} \p^\top \bell_t} \p^\top \bpsi(\u)
    ~ > ~
    \inf_{\u\in\C} \p^\top \bpsi(\u)
    \,.
\]
If $\bpsi$ is $(\L,\Delta_n)$-calibrated, we simply say $\bpsi$ is $\L$-calibrated.
\end{defn}

The above definition of calibration clearly generalizes that used in the binary case. For example, in the case of binary 0-1 classification, with $n=k=2$ and $\L^\zo = \big[ \begin{smallmatrix} 0 & 1 \\ 1 & 0 \end{smallmatrix} \big]$, the probability simplex $\Delta_2$ is equivalent to the interval $[0,1]$, and it can be shown that one only need consider surrogates on the real line, $\C=\R$, and the predictor `$\sign$'; in this case one recovers the familiar definition of binary classification calibration of \cite{Bartlett+06}, namely that a surrogate $\psi:\R\>\R_+^2$ is $(\L^\zo,\Delta_2)$-calibrated if 
\[
\forall p\in[0,1], p\neq \half:
	~~~~\inf_{u\in\R: \sign(u)\neq \sign(p-\half)} p\, \psi_1(u) + (1-p) \psi_2(u)
    ~ > ~
    \inf_{u\in\R}  p\, \psi_1(u) + (1-p)\psi_2(u)
    \,.
\]
Similarly, in the case of binary cost-sensitive classification, with 
$n=k=2$ and $\L^c = \big[ \begin{smallmatrix} 0~ & 1-c \\ c & 0 \end{smallmatrix} \big]$ where $c\in(0,1)$ is the cost of a false positive and $(1-c)$ that of a false negative, one recovers the corresponding definition of \cite{Scott12}, namely that a surrogate $\psi:\R\>\R_+^2$ is $(\L^c,\Delta_2)$-calibrated if 
\[
\forall p\in[0,1], p\neq c:
	~~~~\inf_{u\in\R: \sign(u)\neq \sign(p-c)} p\, \psi_1(u) + (1-p) \psi_2(u)
    ~ > ~
    \inf_{u\in\R}  p\, \psi_1(u) + (1-p)\psi_2(u)
    \,.
\]

The following lemma gives a characterization of calibration similar to that used by \cite{TewariBa07}:
\begin{lem}
\label{lem:pred-pred'}
Let $\L\in\R_+^{n\times k}$ and $\cP\subseteq \Delta_n$. 
Then a surrogate loss $\bpsi:\C\>\R_+^n$ is $(\L,\cP)$-calibrated iff there exists a function $\pred':\S_\psi\>[k]$ such that
\[
\forall \p\in\cP:
    ~~~~\inf_{\z\in\S_\psi:\pred'(\z)\notin\argmin_{t} \p^\top \bell_t} \p^\top \z
    ~ > ~
    \inf_{\z\in\S_\psi} \p^\top \z
    \,.
\]
\end{lem}

In this paper, we will mostly be concerned with $(\L,\Delta_n)$-calibration, which as noted above, we refer to as simply $\L$-calibration. The following result, whose proof is a straightforward generalization of that of a similar result for the 0-1 loss given by \cite{TewariBa07}, explains why $\L$-calibration is useful:
\begin{thm}
\label{thm:class-calibration-consistency}
Let $\L\in\R_+^{n\times k}$.
A surrogate loss $\bpsi:\C\>\R_+^n$ is $\L$-calibrated iff there exists a function $\pred:\C\>[k]$ such that for all distributions $D$ on $\X\times[n]$ 
and all sequences of (vector) functions $\f_m:\X\>\C$,
\[
\er_D^\psi[\f_m] \longrightarrow \er_D^{\psi,*} 
	~~~~~ \textrm{implies} ~~~~~
	\er_D^\L[\pred \circ \f_m] \longrightarrow \er_D^{\L,*}
    \,.
\]
\end{thm}

In particular, \Thm{thm:class-calibration-consistency} implies that a surrogate $\bpsi$ is $\L$-calibrated if and only if $\exists$ a mapping $\pred:\C\>[k]$ such that any $\psi$-consistent algorithm learning models of the form $\f_m:\X\>\C$ (from i.i.d.\ examples $(X_1,Y_1),\ldots,(X_m,Y_m)$) yields an $\L$-consistent algorithm learning models of the form $(\pred \circ \f_m): \X\>[k]$.

\subsection{Trigger Probabilities and Positive Normals} 
\label{subsec:trigger}

Our goal is to study conditions under which a surrogate loss $\bpsi:\C\>\R_+^n$ is $\L$-calibrated for a target loss matrix $\L\in\R_+^{n\times k}$. To this end, we will now define certain properties of both multiclass loss matrices $\L$ and multiclass surrogates $\bpsi$ that will be useful in relating the two. Specifically, we will define \emph{trigger probability sets} associated with a multiclass loss matrix $\L$, and \emph{positive normal sets} associated with a multiclass surrogate $\bpsi$; in \Sec{sec:conditions} we will use these to obtain both necessary and sufficient conditions for calibration.

\begin{defn}[Trigger probability sets]
\label{def:trigger-prob}
Let $\L\in\R_+^{n\times k}$. For each $t\in[k]$, the \emph{trigger probability set of $\L$ at $t$} is defined as
\[
\Q^\L_t
    ~ \= ~
    \Big\{ \p\in\Delta_n: \p^\top (\bell_{t} - \bell_{t'}) \leq 0 ~\forall t'\in[k] \Big\}
    ~ = ~
    \Big\{ \p\in\Delta_n: t \in \argmin_{t'\in[k]} \p^\top \bell_{t'} \Big\}
    \,.
\]
\end{defn}

In words, the trigger probability set $\Q^\L_t$ is the set of class probability vectors for which predicting $t$ is optimal in terms of minimizing $\L$-risk. Such sets have also been studied by \cite{LambertSh09} and \cite{OBrien+08} in a different context. \cite{LambertSh09} show that these sets form what is called a power diagram, which is a generalization of the Voronoi diagram. Trigger probability sets for the 0-1, ordinal regression, and `abstain' loss matrices (described in Examples~\ref{exmp:loss-0-1}, \ref{exmp:loss-ord} and \ref{exmp:loss-abstain}) are illustrated in \Fig{fig:trigger-prob}; the corresponding calculations can be found in \App{app:calculations-trigger-prob}.

\begin{figure}[t]
${\underset{\begin{minipage}{1.8in}
\begin{tiny}
\vspace{0.2cm}
    $Q^\zo_1 = \{ \p\in\Delta_3: p_1\geq\max(p_2,p_3) \}$ \\[1pt]
    $Q^\zo_2 = \{ \p\in\Delta_3: p_2\geq\max(p_1,p_3) \}$ \\[1pt]
    $Q^\zo_3 = \{ \p\in\Delta_3: p_3\geq\max(p_1,p_2) \}$ \\[6pt]
    \begin{center}{\normalsize{(a)}}\end{center}
\end{tiny}
\end{minipage}}{\includegraphics[width=0.3\textwidth]{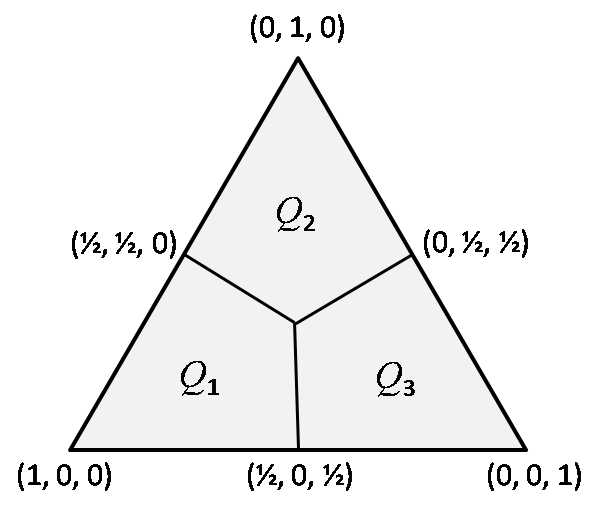}}}$
\hspace{1.5em}
${\underset{\begin{minipage}{1.8in}
\begin{tiny}
\vspace{0.2cm}
    $Q^\ord_1 = \{ \p\in\Delta_3: p_1\geq\half \}$ \\[1pt]
    $Q^\ord_2 = \{ \p\in\Delta_3: p_1\leq\half, p_3\leq\half \}$ \\[1pt]
    $Q^\ord_3 = \{ \p\in\Delta_3: p_3\geq\half \}$ \\[5.5pt]
    \begin{center}{\normalsize{(b)}}\end{center}
\end{tiny}
\end{minipage}}{\includegraphics[width=0.3\textwidth]{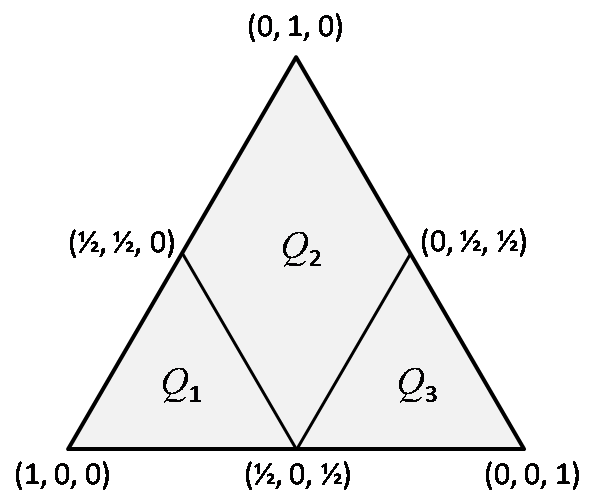}}}$
\hspace{1.5em}
${\underset{\begin{minipage}{1.8in}
\begin{tiny}
\vspace{0.2cm}
    $Q^{(?)}_1 = \{ \p\in\Delta_3: p_1\geq\half \}$ \\[1pt]
    $Q^{(?)}_2 = \{ \p\in\Delta_3: p_2\geq\half \}$ \\[1pt]
    $Q^{(?)}_3 = \{ \p\in\Delta_3: p_3\geq\half \}$ \\[1pt]
    $Q^{(?)}_4 = \{ \p\in\Delta_3: \max(p_1,p_2,p_3)\leq\half \}$
    \vspace{-10pt}
    \begin{center}{\normalsize{(c)}}\end{center}
\end{tiny}
\end{minipage}}{\includegraphics[width=0.3\textwidth]{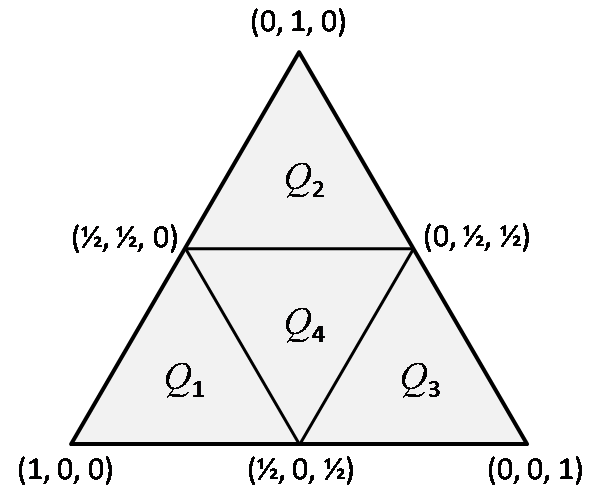}}}$
\vspace{-8pt}
\caption{Trigger probability sets for (a) 0-1 loss $\L^\zo$; (b) ordinal regression loss $\L^\ord$; and (c) `abstain' loss $\L^\abstain$; all for $n=3$, for which the probability simplex can be visualized easily. Calculations of these sets can be found in \App{app:calculations-trigger-prob}.}
\label{fig:trigger-prob}
\end{figure}

\begin{defn}[Positive normal sets]
\label{def:pos-normals}
Let $\bpsi:\C\>\R_+^n$. For each point $\z\in\S_\psi$, the \emph{positive normal set of $\bpsi$ at $\z$} is defined as\footnote{For points $\z$ in the interior of $\S_\psi$, $\cN^\psi(\z)$ is empty.}
\[
\cN^\psi(\z)
    ~ \= ~
    \Big\{ \p\in\Delta_n: \p^\top (\z-\z') \leq 0 ~\forall \z'\in\S_\psi \Big\}
    ~ = ~
    \Big\{ \p\in\Delta_n: \p^\top \z = \inf_{\z'\in\S_\psi} \p^\top\z' \Big\}
    \,.
\]
For any sequence of points $\{\z_m\}$ in $\S_\psi$, the \emph{positive normal set of $\bpsi$ at $\{\z_m\}$} is defined as\footnote{For sequences $\{\z_m\}$ for which $\lim_{m\>\infty} \p^\top \z_m$ does not exist for any $\p$, $\cN^\psi(\z)$ is empty.}
\[
\cN^\psi(\{\z_m\})
    ~ \= ~
    \Big\{ \p\in\Delta_n: \lim_{m\>\infty} \p^\top \z_m = \inf_{\z'\in\S_\psi} \p^\top \z' \Big\}
    \,.
\]
\end{defn}

In words, the positive normal set $\cN^\psi(\z)$ at a point $\z=\bpsi(\u) \in \cR_\psi$ is the set of class probability vectors for which predicting $\u$ is optimal in terms of minimizing $\psi$-risk. Such sets were also studied by \cite{TewariBa07}. The extension to sequences of points in $\S_\psi$ is needed for technical reasons in some of our proofs. Note that for $\cN^\psi(\{\z_m\})$ to be well-defined, the sequence $\{\z_m\}$ need not converge itself; however if the sequence $\{\z_m\}$ does converge to some point $\z\in\S_\psi$, then $\cN^\psi(\{\z_m\})=\cN^\psi(\z)$. Positive normal sets for the Crammer-Singer surrogate (described in \Exmp{exmp:surrogate-CS}) at 4 points are illustrated in \Fig{fig:pos-normals}; the corresponding calculations can be found in \App{app:calculations-pos-normals}.

\begin{figure}[t]
\begin{center}
\includegraphics[width=0.3\textwidth]{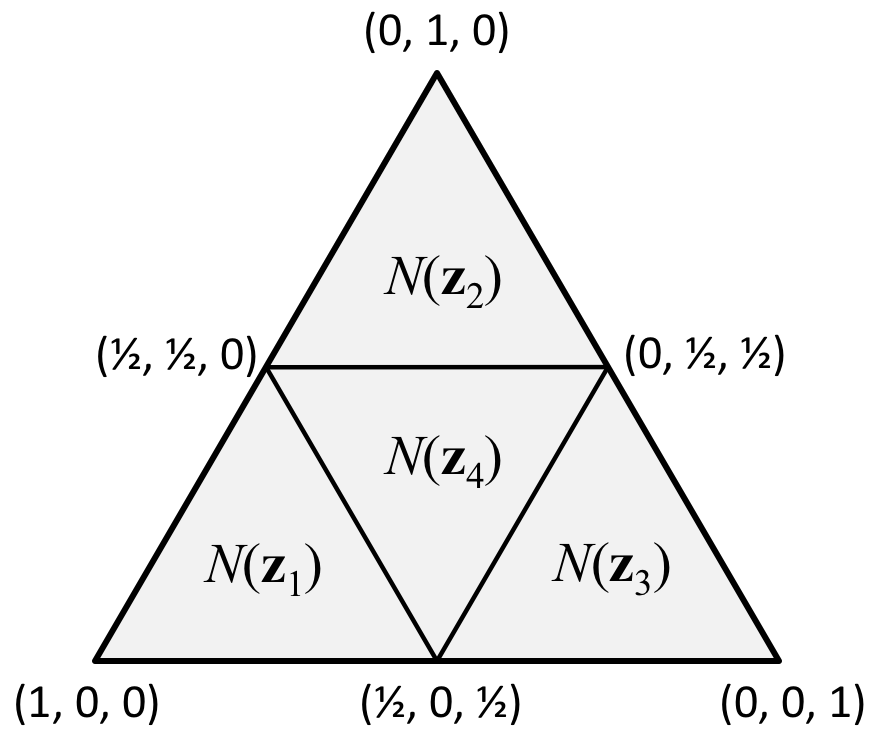}
\hspace{1.5em}
\begin{minipage}{2.3in}
\begin{scriptsize}
    $\cN^\CS(\z_1) = \{ \p\in\Delta_3: p_1\geq\half \}$ \\[1pt]
    $\cN^\CS(\z_2) = \{ \p\in\Delta_3: p_2\geq\half \}$ \\[1pt]
    $\cN^\CS(\z_3) = \{ \p\in\Delta_3: p_3\geq\half \}$ \\[1pt]
    $\cN^\CS(\z_4) = \{ \p\in\Delta_3: \max(p_1,p_2,p_3)\leq\half \}$
    \vspace{4cm}
\end{scriptsize}
\end{minipage}
\vspace{-2.75cm}
\caption{Positive normal sets for the Crammer-Singer surrogate $\bpsi^\CS$ for $n=3$, at 4 points $\z_i = \bpsi^\CS(\u_i) \in \R_+^3$ ($i\in[4]$) for $\u_1 = (1,0,0)^\top$, $\u_2 = (0,1,0)^\top$, $\u_3 = (0,0,1)^\top$, and $\u_4 = (0,0,0)^\top$. Calculations of these sets are based on \Lem{lem:normal-computation} and can be found in \App{app:calculations-pos-normals}.}
\label{fig:pos-normals}
\end{center}
\end{figure}

\section{Conditions for Calibration}
\label{sec:conditions}

In this section we give both necessary conditions (\Sec{subsec:conditions-nec}) and sufficient conditions (\Sec{subsec:conditions-suff}) for a surrogate $\bpsi$ to be calibrated w.r.t.\ an arbitrary target loss matrix $\L$. Both sets of conditions involve the trigger probability sets of $\L$ and the positive normal sets of $\bpsi$; in \Sec{subsec:computation-pos-normals} we give a result that facilitates computation of positive normal sets for certain classes of surrogates $\bpsi$.

\subsection{Necessary Conditions for Calibration} 
\label{subsec:conditions-nec}


\begin{figure}[t]
\begin{center}
\includegraphics[width=0.32\textwidth]{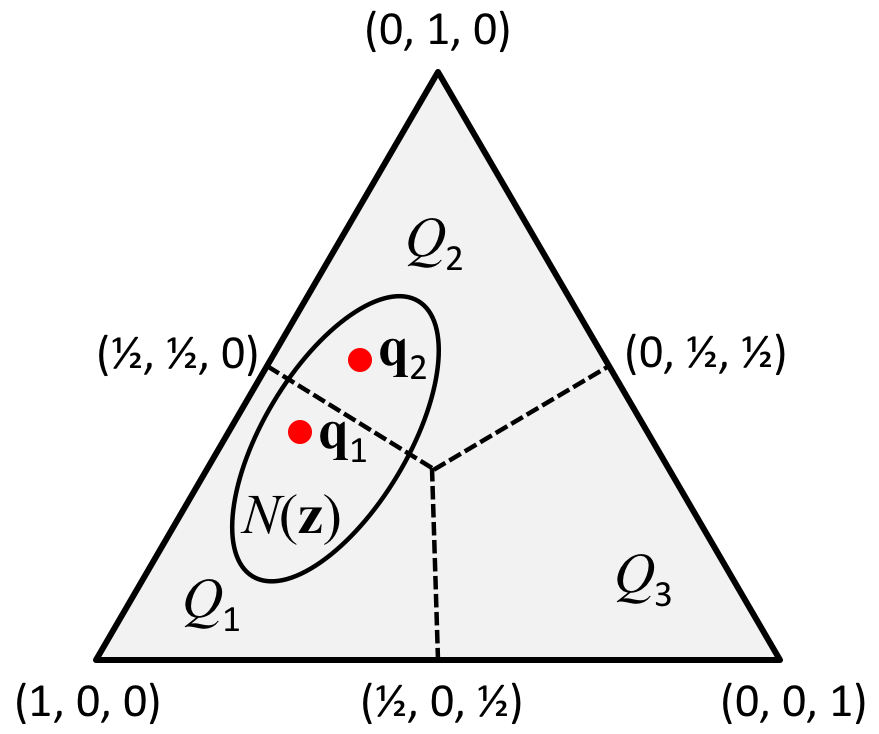}
\caption{Visual proof of \Thm{thm:condition-necessary}. If a surrogate $\bpsi$ is such that its positive normal set $\cN^\psi(\z)$ at some point $\z$ has non-empty intersection with the interiors of two trigger probability sets (say $\Q^\L_{1}$ and $\Q^\L_{2}$) of $\L$, then $\bpsi$ cannot be $\L$-calibrated.}
\label{fig:necess-cond-illus}
\end{center}
\end{figure}

We start by deriving necessary conditions for $\L$-calibration of a surrogate loss $\bpsi$. Consider what happens if for some point $\z\in\S_\psi$, the positive normal set of $\bpsi$ at $\z$, $\cN^\psi(\z)$, has a non-empty intersection with the interiors of two trigger probability sets of $\L$, say $\Q^\L_1$ and $\Q^\L_2$ (see \Fig{fig:necess-cond-illus} for an illustration), which means $\exists \q_1, \q_2 \in \cN^\psi(\z)$ with $\argmin_{t\in[k]} \q_1^\top \bell_t =\{1\}$ and $\argmin_{t\in[k]} \q_2^\top \bell_t =\{2\}$. If $\bpsi$ is $\L$-calibrated, then by \Lem{lem:pred-pred'}, we have $\exists \pred':\S_\psi\>[k]$ such that 
\begin{eqnarray*}
\inf_{\z'\in\S_\psi:\pred'(\z')\neq 1} \q_1^\top \z' 
~ = ~
\inf_{\z'\in\S_\psi:\pred'(\z')\notin\argmin_{t} \q_1^\top \bell_t} \q_1^\top \z' 
	& > & 
	\inf_{\z'\in\S_\psi} \q_1^\top \z' 
	~ = ~ 
	\q_1^\top \z
\\
\inf_{\z'\in\S_\psi:\pred'(\z')\neq 2} \q_2^\top \z' 
~ = ~
\inf_{\z'\in\S_\psi:\pred'(\z')\notin\argmin_{t} \q_2^\top \bell_t} \q_2^\top \z' 
	& > & 
	\inf_{\z'\in\S_\psi} \q_2^\top \z' 
	~ = ~
	\q_2^\top \z
	\,. 
\end{eqnarray*}
The first inequality above implies $\pred'(\z)=1$; the second inequality implies $\pred'(\z)=2$, leading to a contradiction. This gives us the following necessary condition for $\L$-calibration of $\bpsi$, which requires the positive normal sets of $\bpsi$ at all points $\z\in\S_\psi$ to be `well-behaved' w.r.t.\ $\L$ in the sense of being contained within individual trigger probability sets of $\L$ and generalizes the `admissibility' condition used for 0-1 loss by \cite{TewariBa07}:

\begin{thm}
\label{thm:condition-necessary}
Let $\L\in\R_+^{n\times k}$, and let $\bpsi:\C\>\R_+^n$ be $\L$-calibrated. Then for all points $\z \in\S_\psi$, there exists some $t\in[k]$ such that $\cN^\psi(\z)\subseteq\Q^\L_t$.
\end{thm}
\begin{proof}
See above discussion.
\end{proof}

In fact, we have the following stronger necessary condition, which requires the positive normal sets of $\bpsi$ not only at all points $\z\in\S_\psi$ but also at all sequences $\{\z_m\}$ in $\S_\psi$ to be contained within individual trigger probability sets of $\L$: 

\begin{thm}
\label{thm:condition-necessary-sequence}
Let $\L\in\R_+^{n\times k}$, and let $\bpsi:\C\>\R_+^n$ be $\L$-calibrated. Then for all sequences $\{\z_m\}$ in $\S_\psi$, there exists some $t\in[k]$ such that $\cN^\psi(\{\z_m\})\subseteq\Q^\L_t$.
\end{thm}
\begin{proof}
Assume for the sake of contradiction that there is some sequence $\{\z_m\}$ in $\S_\psi$ for which $\cN^{\psi}(\{\z_m\})$ is not contained in $\Q^\L_t$ for any $t\in[k]$. Then $\forall t\in[k]$, $\exists \q_t \in\cN^{\psi}(\{\z_m\})$ such that $\q_t\notin\Q^\L_t$, i.e.\ such that $t\notin \argmin_{t'} \q_t^\top \bell_{t'}$. 
Now, since $\bpsi$ is $\L$-calibrated, by Lemma \ref{lem:CCequiv-sequence}, there exists a function $\pred':\S_\psi\>[k]$ such that for all $\p\in \cN^{\psi}(\{\z_m\})$, we have $\pred'(\z_m)\in \arg\min_{t'} \p^\top\bell_{t'}$ for all large enough $m$.
In particular, for $\p=\q_t$, we get $\pred'(\z_m)\in \argmin_{t'} \q_t^\top \bell_{t'}$ ultimately. Since this is true for each $t\in[k]$, we get $\pred'(\z_m)\in \cap_{t\in[k]} \argmin_{t'} \q_t^\top \bell_{t'}$ ultimately. However by choice of $\q_t$, this intersection is empty, thus yielding a contradiction. This completes the proof.  
\end{proof}

Note that \Thm{thm:condition-necessary-sequence} includes \Thm{thm:condition-necessary} as a special case, since $\cN^\psi(\z) = \cN^\psi(\{\z_m\})$ for the constant sequence $\z_m = \z ~\forall m$. We stated \Thm{thm:condition-necessary} separately above since it had a simple, direct proof that helps build intuition. 

\begin{exmp}[Crammer-Singer surrogate is not calibrated for 0-1 loss]
\label{exmp:calibration-CS-0-1}
Looking at the positive normal sets of the Crammer-Singer surrogate $\bpsi^\CS$ (for $n=3$) shown in \Fig{fig:pos-normals} and the trigger probability sets of the 0-1 loss $\L^\zo$ shown in \Fig{fig:trigger-prob}(a), we see that $\cN^\CS(\z_4)$ is not contained in any single trigger probability set of $\L^\zo$, and therefore applying \Thm{thm:condition-necessary}, it is immediately clear that $\bpsi^\CS$ is not $\L^\zo$-calibrated (this was also established by \cite{TewariBa07} and \cite{Zhang04b}). 
\end{exmp}

\subsection{Sufficient Condition for Calibration} 
\label{subsec:conditions-suff}

We now give a sufficient condition for $\L$-calibration of a surrogate loss $\bpsi$ that will be helpful in showing calibration of various surrogates. In particular, we show that for a surrogate loss $\bpsi$ to be $\L$-calibrated, it is sufficient for the above property of positive normal sets of $\bpsi$ being contained in trigger probability sets of $\L$ to hold for only a finite number of points in $\S_\psi$, as long as the corresponding positive normal sets jointly cover $\Delta_n$:

\begin{thm}
\label{thm:condition-sufficient}
Let $\L\in\R_+^{n\times k}$ and $\bpsi:\C\>\R_+^n$. Suppose there exist $r\in\Z_+$ and $\z_1,\ldots,\z_r \in \S_\psi$ such that $\bigcup_{j=1}^r \cN^\psi(\z_j) = \Delta_n$ and for each $j\in[r]$, $\exists t\in[k]$ such that $\cN^\psi(\z_j) \subseteq \Q^\L_t$. Then $\bpsi$ is $\L$-calibrated.
\end{thm}


\begin{exmp}[Crammer-Singer surrogate is calibrated for $\L^{\abstain}$ and $\L^{\ord}$ for $n=3$]
\label{exmp:calibration-CS-abstain-ord}
Inspecting the positive normal sets of the Crammer-Singer surrogate $\bpsi^\CS$ (for $n=3)$ in \Fig{fig:pos-normals} and the trigger probability sets of the `abstain' loss matrix $\L^{\abstain}$ in \Fig{fig:trigger-prob}(c), we see that $\cN^\CS(\z_i) =\Q^{\abstain}_i ~\forall i\in[4]$, and therefore by \Thm{thm:condition-sufficient}, the Crammer-Singer surrogate $\bpsi^\CS$ is $\L^{\abstain}$-calibrated. Similarly, looking at the trigger probability sets of the ordinal regression loss matrix $\L^\ord$ in \Fig{fig:trigger-prob}(b) and again applying \Thm{thm:condition-sufficient}, we see that the Crammer-Singer surrogate $\bpsi^\CS$ is also $\L^\ord$-calibrated!
\end{exmp}

Some additional examples of applications of Theorems~\ref{thm:condition-necessary} and \ref{thm:condition-sufficient} are provided in \Sec{subsec:computation-pos-normals} below. Both the necessary and sufficient conditions above will also be used when we study the convex calibration dimension of a loss matrix $\L$ in \Sec{sec:ccdim}.

\subsection{Computation of Positive Normal Sets}
\label{subsec:computation-pos-normals}

Both the necessary and sufficient conditions for calibration above involve the positive normal sets $\cN^\psi(\z)$ at various points $\z\in \S_\psi$. Thus in order to use the above results to show that a surrogate $\bpsi$ is (or is not) $\L$-calibrated, one needs to be able to compute or characterize the sets $\cN^\psi(\z)$. Here we give a method for computing these sets for certain surrogates $\bpsi$ at certain points $\z\in\S_\psi$. Specifically, the following result gives an explicit method for computing $\cN^\psi(\z)$ for convex surrogate losses $\bpsi$ operating on a convex surrogate space $\C\subseteq\R^d$, at points $\z=\bpsi(\u) \in\cR_\psi$ for which the subdifferential $\partial \psi_y(\u)$ for each $y\in[n]$ can be described as the convex hull of a finite number of points in $\R^d$; this is particularly applicable for piecewise linear surrogates.\footnote{Recall that a vector function is convex if all its component functions are convex.}${}^{,}$\footnote{Recall that the subdifferential of a convex function $\phi:\R^d\>\bbar{\R}$ at a point $\u_0\in\R^d$ is defined as
$\partial \phi(\u_0) = \big\{ \w\in\R^d: \phi(\u)-\phi(\u_0) \geq \w^\top(\u-\u_0) ~\forall \u\in\R^d \big\}$
and is a convex set in $\R^d$ (e.g.\ see \cite{Bertsekas+03}).}

\begin{lem}
\label{lem:normal-computation}
Let $\C\subseteq\R^d$ be a convex set and let $\bpsi:\C\>\R_+^n$ be convex.
Let $\z=\bpsi(\u)$ for some $\u\in\C$ such that $\forall y\in[n]$, the subdifferential of $\psi_y$ at $\u$ can be written as 
\[
\partial \psi_y(\u) = \conv\big( \{\w^y_1,\ldots,\w^y_{s_y}\} \big)
\]
for some $s_y \in\Z_+$ and $\w^y_1,\ldots,\w^y_{s_y} \in\R^d$.
Let $s=\sum_{y=1}^n s_y$, and let
\[
\A
    =
    \left[ \w^1_1 \ldots \w^1_{s_1} \w^2_1 \ldots \w^2_{s_2} \ldots \ldots \w^n_1 \ldots \w^n_{s_n} \right] \in \R^{d\times s}
    \,;
\hspace{0.75cm}
\B
    =
    \left[ b_{yj} \right] \in \R^{n\times s}
    \,,
\]
where $b_{yj}$ is $1$ if the $j$-th column of $\A$ came from $\{\w^y_1,\ldots,\w^y_{s_y}\}$ and $0$ otherwise.
Then
\vspace{-2pt}
\[
\cN^\psi(\z)
    =
    \Big\{ \p\in\Delta_n: \p = \B\q ~\mbox{for some $\q \in \Null(\A)\cap\Delta_s$} \Big\}
    \,,
\]
where $\Null(\A)\subseteq\R^s$ denotes the null space of the matrix $\A$.
\end{lem}

The proof makes use of the fact that a convex function $\phi:\R^d\>\R$ attains its minimum at $\u_0\in\R^d$ iff the subdifferential $\partial \phi(\u_0)$ contains $\0\in\R^d$ (e.g.\ see \cite{Bertsekas+03}). We will also make use of the fact that if $\phi_1,\phi_2:\R^d\>\R$ are convex functions, then the subdifferential of their sum $\phi_1+\phi_2$ at $\u_0$ is is equal to the Minkowski sum of the subdifferentials of $\phi_1$ and $\phi_2$ at $\u_0$:
\[
\partial(\phi_1+\phi_2)(\u_0)
    ~ = ~
    \big\{ \w_1 + \w_2: \w_1\in\partial\phi_1(\u_0), \w_2\in\partial\phi_2(\u_0) \big\}
    \,.
\]

\begin{proof} (Proof of \Lem{lem:normal-computation}) \\[6pt]
We have for all $\p\in\R^n$,
\begin{eqnarray*}
\p\in\cN^\psi(\bpsi(\u))
    & \Longleftrightarrow &
    \p\in\Delta_n, ~ \p^\top \bpsi(\u) \leq \p^\top \z' ~\forall \z'\in\S_\psi
\\
    & \Longleftrightarrow &
    \p\in\Delta_n, ~ \p^\top \bpsi(\u) \leq \p^\top \z' ~\forall \z'\in\cR_\psi
\\
    & \Longleftrightarrow &
    \p\in\Delta_n, ~ \mbox{and the convex function $\phi(\u') = \p^\top \bpsi(\u') = \sum_{y=1}^n p_y \psi_y(\u')$}
\\
    & &
    \mbox{achieves its minimum at $\u' = \u$}
\\
    & \Longleftrightarrow &
    \p\in\Delta_n, ~ \0 \in \sum_{y=1}^n p_y \partial\psi_y(\u)
\\
    & \Longleftrightarrow &
    \p\in\Delta_n, ~ \0 = \sum_{y=1}^n p_y \sum_{j=1}^{s_y} v^y_j \w^y_j ~ \mbox{for some $\v^y \in \Delta_{s_y}$}
\\
    & \Longleftrightarrow &
    \p\in\Delta_n, ~ \0 = \sum_{y=1}^n \sum_{j=1}^{s_y} q^y_j \w^y_j ~ \mbox{for some $\q^y = p_y\v^y$, ~ $\v^y \in \Delta_{s_y}$}
\\
    & \Longleftrightarrow &
    \p\in\Delta_n, \A\q = \0 ~ \mbox{for some $\q=(p_1\v^1,\ldots,p_n\v^n)^\top\in\Delta_s$, ~ $\v^y\in\Delta_{s_y}$}
\\
    & \Longleftrightarrow &
    \p = \B\q ~ \mbox{for some $\q\in\Null(\A)\cap\Delta_s$}
    \,.
\end{eqnarray*}
\end{proof}

We now give examples of computation of positive normal sets using \Lem{lem:normal-computation} for two convex surrogates, both of which operate on the one-dimensional surrogate space $\C=\R$, and as we shall see, turn out to be calibrated w.r.t.\ the ordinal regression loss $\L^\ord$ but not w.r.t.\ the 0-1 loss $\L^\zo$ or the `abstain' loss $\L^\abstain$. As another example of an application of \Lem{lem:normal-computation}, calculations showing computation of positive normal sets of the Crammer-Singer surrogate (as shown in \Fig{fig:pos-normals}) are given in the appendix.

\begin{exmp}[Positive normal sets of `absolute' surrogate]
\label{exmp:pos-normals-abs-surrogate}
Let $n=3$, and let $\C=\R$. Consider the `absolute' surrogate $\bpsi^\abs:\R\>\R_+^3$ defined as follows:
\begin{equation}
\psi^\abs_y(u)
    =
    |u-y| \qquad   \forall y\in[3], \, u\in\R
    \,.
\end{equation}
Clearly, $\bpsi^\abs$ is a convex function (see \Fig{fig:abs-surrogate}).
Moreover, we have 
\[
\cR_\abs ~ = ~ \bpsi^\abs(\R) ~ = ~ \big\{ ( |u-1|, |u-2|, |u-3| )^\top: u \in \R \big\} \subset \R_+^3
	\,.
\]
Now let $u_1 =1$, $u_2 = 2$, and $u_3 = 3$, and let
\begin{eqnarray*}
\z_1 
	& = &
	\bpsi^\abs(u_1) ~ = ~ \bpsi^\abs(1) ~  = ~ (0,1,2)^\top \in \cR_\abs
\\
\z_2 
	& = &
	\bpsi^\abs(u_2) ~ = ~ \bpsi^\abs(2) ~  = ~  (1,0,1)^\top \in \cR_\abs
\\
\z_3 
	& = & 
	\bpsi^\abs(u_3) ~ = ~ \bpsi^\abs(3) ~  = ~  (2,1,0)^\top \in \cR_\abs
	\,.
\end{eqnarray*}
Let us consider computing the positive normal sets of $\bpsi^\abs$ at the 3 points $\z_1,\z_2,\z_3$ above. 
To see that $\z_1$ satisfies the conditions of \Lem{lem:normal-computation}, note that 
\begin{eqnarray*}
\partial\psi^\abs_1(u_1) ~ = ~ \partial\psi^\abs_1(1) & = & [-1,1] = \conv(\{+1,-1\}) \,; \\
\partial\psi^\abs_2(u_1)  ~ = ~ \partial\psi^\abs_2(1) & = & \{-1\}  = \conv(\{-1\}) \,; \\
\partial\psi^\abs_3(u_1) ~ = ~ \partial\psi^\abs_3(1) & = & \{-1\}  = \conv(\{-1\}) \,.
\end{eqnarray*}
Therefore, we can use \Lem{lem:normal-computation} to compute $\cN^\abs(\z_1)$. Here $s=4$, and
\[
\A = \left[ \begin{array}{cccc}
        +1 & -1 & -1 & -1
    \end{array} \right]
    \,;
\hspace{0.5cm}
\B = \left[ \begin{array}{cccc}
        1 & 1 & 0 & 0 \\
        0 & 0 & 1 & 0 \\
        0 & 0 & 0 & 1
    \end{array} \right]
    \,.
\]
This gives
\begin{eqnarray*}
\cN^\abs(\z_1)
    & = &
    \big\{ \p\in\Delta_3: \p = (q_1 + q_2, q_3, q_4) ~\mbox{for some $\q\in\Delta_4, ~ q_1 - q_2 - q_3 - q_4 = 0$} \big\}
\\
    & = &
    \big\{ \p\in\Delta_3: \p = (q_1 + q_2, q_3, q_4) ~\mbox{for some $\q\in\Delta_4, ~ q_1 = \textstyle{\half}$} \big\}
\\
    & = &
    \big\{ \p\in\Delta_3: p_1 \geq \textstyle{\half} \big\}
    \,.
\end{eqnarray*}
It is easy to see that $\z_2$ and $\z_3$ also satisfy the conditions of \Lem{lem:normal-computation}; similar computations then yield
\begin{eqnarray*}
\cN^\abs(\z_2) 
	& = &     
	\big\{ \p\in\Delta_3: p_1 \leq \textstyle{\half}, p_3 \leq \textstyle{\half} \big\}
\\
\cN^\abs(\z_3) 
	& = &     
	\big\{ \p\in\Delta_3: p_3 \geq \textstyle{\half} \big\}
	\,.
\end{eqnarray*}
The positive normal sets above are shown in \Fig{fig:abs-surrogate}. Comparing these with the trigger probability sets in \Fig{fig:trigger-prob}, we have by \Thm{thm:condition-sufficient} that $\bpsi^\abs$ is $\L^\ord$-calibrated, and by \Thm{thm:condition-necessary} that $\bpsi^\abs$ is not calibrated w.r.t.\ $\L^\zo$ or $\L^\abstain$.
\end{exmp}

\begin{figure}[t]
\begin{center}
$\underset{\text{\normalsize{(a)}}}{\includegraphics[width=0.35 \textwidth]{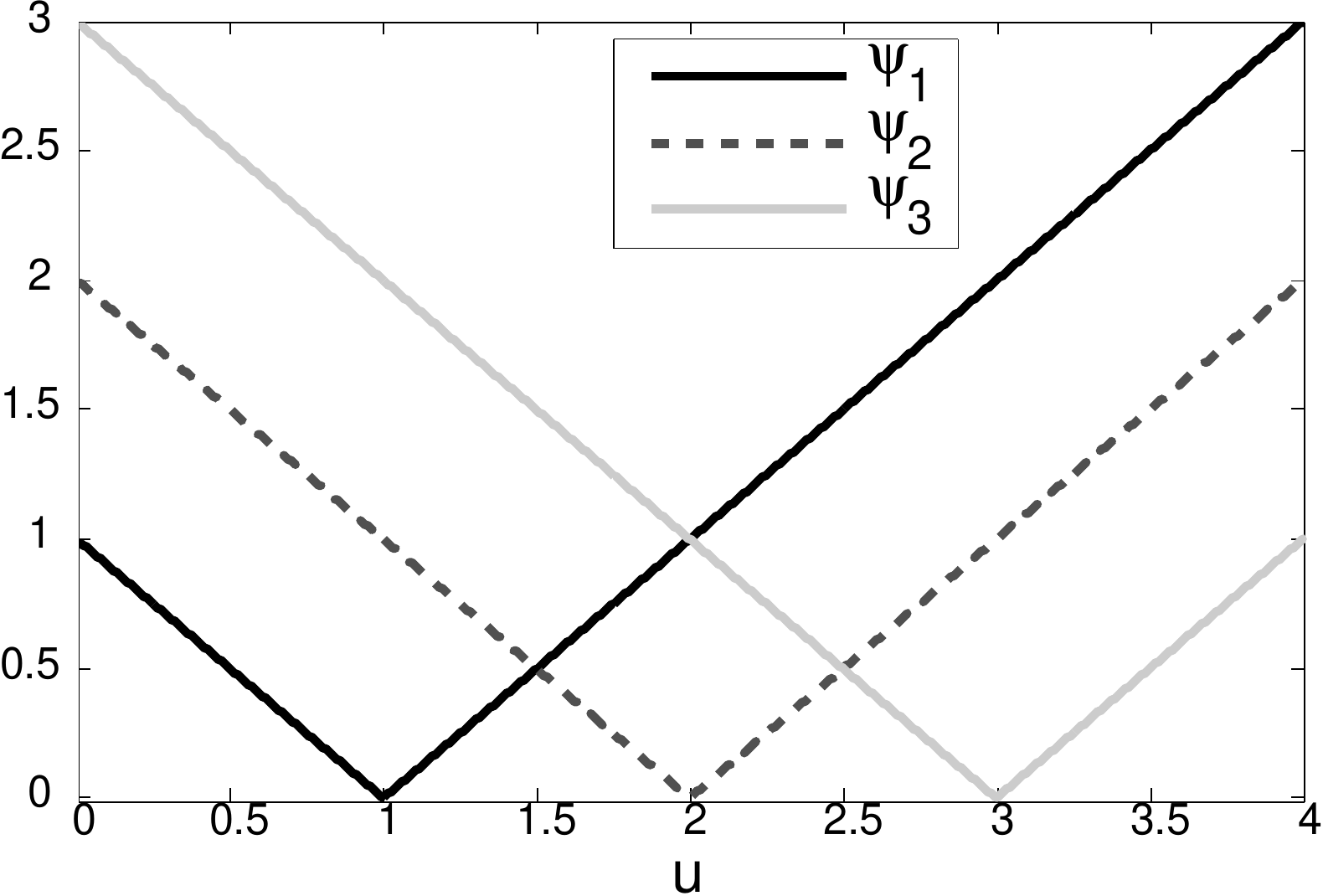}}$
\hspace{1cm}
$\underset{\text{\normalsize{(b)}}}{\includegraphics[width=0.3 \textwidth]{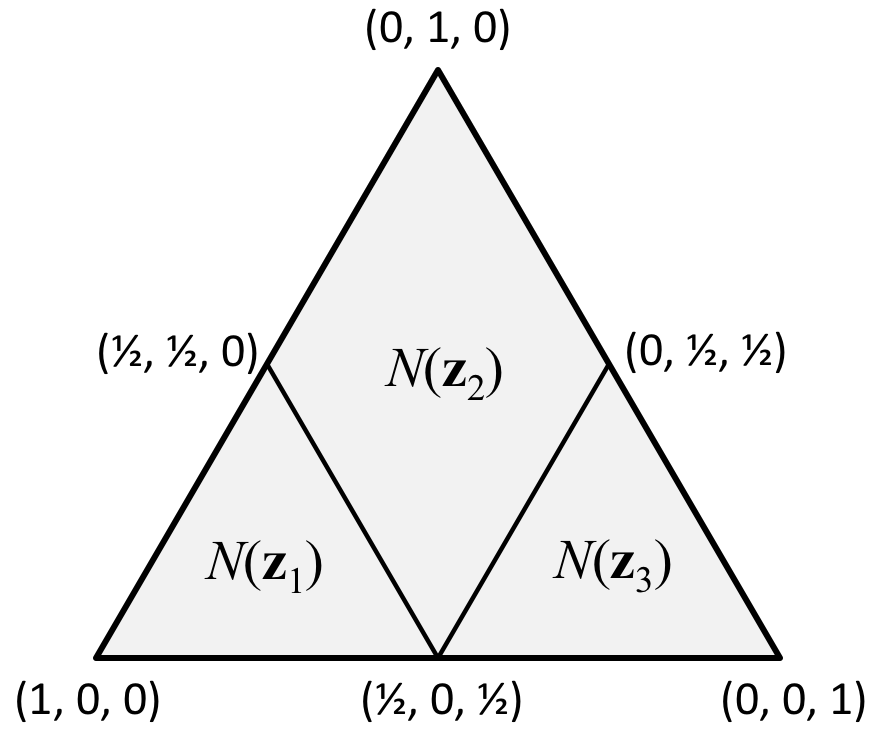}}$
\caption{(a) The `absolute' surrogate $\bpsi^\abs:\R\>\R_+^3$ (for $n=3$), and (b) its positive normal sets at 3 points $\z_i = \bpsi^\abs(u_i)\in\R_+^3$ ($i\in[3]$) for $u_1 = 1, u_2 = 2, u_3 = 3$. See \Exmp{exmp:pos-normals-abs-surrogate} for details.}
\label{fig:abs-surrogate}
\end{center}
\end{figure}

\begin{exmp}[Positive normal sets of `$\epsilon$-insensitive' surrogate]
\label{exmp:pos-normals-eps-surrogate}
Let $n=3$, and let $\C=\R$. Let $\epsilon \in[0,0.5)$, and consider the `$\epsilon$-insensitive' surrogate $\bpsi^\epsilon:\R\>\R_+^3$ defined as follows:
\begin{equation}
\psi^\epsilon_y(u)
    =
    \big( |u-y| - \epsilon \big)_+ \qquad   \forall y\in[3], \, u\in\R
    \,.
\end{equation}
For $\epsilon = 0$, we have $\bpsi^\epsilon = \bpsi^\abs$.
Clearly, $\bpsi^\epsilon$ is a convex function (see \Fig{fig:eps-surrogate}).
Moreover, we have 
\[
\cR_\epsilon 
	~ = ~ \bpsi^\epsilon(\R) 
	~ = ~ \big\{ \big( (|u-1|-\epsilon)_+, (|u-2|-\epsilon)_+, (|u-3|-\epsilon)_+ \big)^\top: u \in \R \big\} \subset \R_+^3
	\,.
\]
For concreteness, we will take $\epsilon = 0.25$ below, but similar computations hold $\forall\epsilon \in (0,0.5)$. 
Let $u_1 =1+\epsilon = 1.25$, $u_2 = 2-\epsilon = 1.75$, $u_3 = 2+\epsilon = 2.25$, and $u_4 = 3-\epsilon = 2.75$, and let
\begin{eqnarray*}
\z_1 
	& = &
	\bpsi^{0.25}(u_1) ~ = ~ \bpsi^{0.25}(1.25) ~  = ~ (0,0.5,1.5)^\top \in \cR_{0.25}
\\
\z_2 
	& = &
	\bpsi^{0.25}(u_2) ~ = ~ \bpsi^{0.25}(1.75) ~  = ~  (0.5,0,1)^\top \in \cR_{0.25}
\\
\z_3 
	& = & 
	\bpsi^{0.25}(u_3) ~ = ~ \bpsi^{0.25}(2.25) ~  = ~  (1,0,0.5)^\top \in \cR_{0.25}
\\
\z_4
	& = & 
	\bpsi^{0.25}(u_4) ~ = ~ \bpsi^{0.25}(2.75) ~  = ~  (1.5,0.5,0)^\top \in \cR_{0.25}
	\,.
\end{eqnarray*}
Let us consider computing the positive normal sets of $\bpsi^{0.25}$ at the 4 points $\z_i$ ($i\in[4]$) above. 
To see that $\z_1$ satisfies the conditions of \Lem{lem:normal-computation}, note that 
\begin{eqnarray*}
\partial\psi^{0.25}_1(u_1) ~ = ~ \partial\psi^{0.25}_1(1.25) & = & [0,1] = \conv(\{0,1\}) \,; \\
\partial\psi^{0.25}_2(u_1)  ~ = ~ \partial\psi^{0.25}_2(1.25) & = & \{-1\}  = \conv(\{-1\}) \,; \\
\partial\psi^{0.25}_3(u_1) ~ = ~ \partial\psi^{0.25}_3(1.25) & = & \{-1\}  = \conv(\{-1\}) \,.
\end{eqnarray*}
Therefore, we can use \Lem{lem:normal-computation} to compute $\cN^{0.25}(\z_1)$. Here $s=4$, and
\[
\A = \left[ \begin{array}{cccc}
        0 & 1 & -1 & -1
    \end{array} \right]
    \,;
\hspace{0.5cm}
\B = \left[ \begin{array}{cccc}
        1 & 1 & 0 & 0 \\
        0 & 0 & 1 & 0 \\
        0 & 0 & 0 & 1
    \end{array} \right]
    \,.
\]
This gives
\begin{eqnarray*}
\cN^{0.25}(\z_1)
    & = &
    \big\{ \p\in\Delta_3: \p = (q_1 + q_2, q_3, q_4) ~\mbox{for some $\q\in\Delta_4, ~ q_2 - q_3 - q_4 = 0$} \big\}
\\
    & = &
    \big\{ \p\in\Delta_3: \p = (q_1 + q_2, q_3, q_4) ~\mbox{for some $\q\in\Delta_4, ~ q_1 + q_2 \geq q_3 + q_4$} \big\}
\\
    & = &
    \big\{ \p\in\Delta_3: p_1 \geq \textstyle{\half} \big\}
    \,.
\end{eqnarray*}
Similarly, to see that $\z_2$ satisfies the conditions of \Lem{lem:normal-computation}, note that 
\begin{eqnarray*}
\partial\psi^{0.25}_1(u_2) ~ = ~ \partial\psi^{0.25}_1(1.75) & = & \{1\} = \conv(\{1\}) \,; \\
\partial\psi^{0.25}_2(u_2)  ~ = ~ \partial\psi^{0.25}_2(1.75) & = & [-1,0]  = \conv(\{-1,0\}) \,; \\
\partial\psi^{0.25}_3(u_2) ~ = ~ \partial\psi^{0.25}_3(1.75) & = & \{-1\}  = \conv(\{-1\}) \,.
\end{eqnarray*}
Again, we can use \Lem{lem:normal-computation} to compute $\cN^{0.25}(\z_2)$; here $s=4$, and
\[
\A = \left[ \begin{array}{cccc}
        1 & -1 & 0 & -1
    \end{array} \right]
    \,;
\hspace{0.5cm}
\B = \left[ \begin{array}{cccc}
        1 & 0 & 0 & 0 \\
        0 & 1 & 1 & 0 \\
        0 & 0 & 0 & 1
    \end{array} \right]
    \,.
\]
This gives
\begin{eqnarray*}
\cN^{0.25}(\z_2)
    & = &
    \big\{ \p\in\Delta_3: \p = (q_1, q_2 + q_3, q_4) ~\mbox{for some $\q\in\Delta_4, ~ q_1 - q_2 - q_4 = 0$} \big\}
\\
    & = &
    \big\{ \p\in\Delta_3: p_1 \geq p_3, ~ p_1 \leq \textstyle{\half} \big\}
    \,.
\end{eqnarray*}
It is easy to see that $\z_3$ and $\z_4$ also satisfy the conditions of \Lem{lem:normal-computation}; similar computations then yield
\begin{eqnarray*}
\cN^{0.25}(\z_3) 
	& = &     
	\big\{ \p\in\Delta_3: p_1 \leq p_3, ~ p_3 \leq \textstyle{\half} \big\}
\\
\cN^{0.25}(\z_4) 
	& = &     
	\big\{ \p\in\Delta_3: p_3 \geq \textstyle{\half} \big\}
	\,.
\end{eqnarray*}
The positive normal sets above are shown in \Fig{fig:eps-surrogate}. Comparing these with the trigger probability sets in \Fig{fig:trigger-prob}, we have by \Thm{thm:condition-sufficient} that $\bpsi^{0.25}$ is $\L^\ord$-calibrated, and by \Thm{thm:condition-necessary} that $\bpsi^{0.25}$ is not calibrated w.r.t.\ $\L^\zo$ or $\L^\abstain$.
\end{exmp}

\begin{figure}[t]
\begin{center}
$\underset{\text{\normalsize{(a)}}}{\includegraphics[width=0.35 \textwidth]{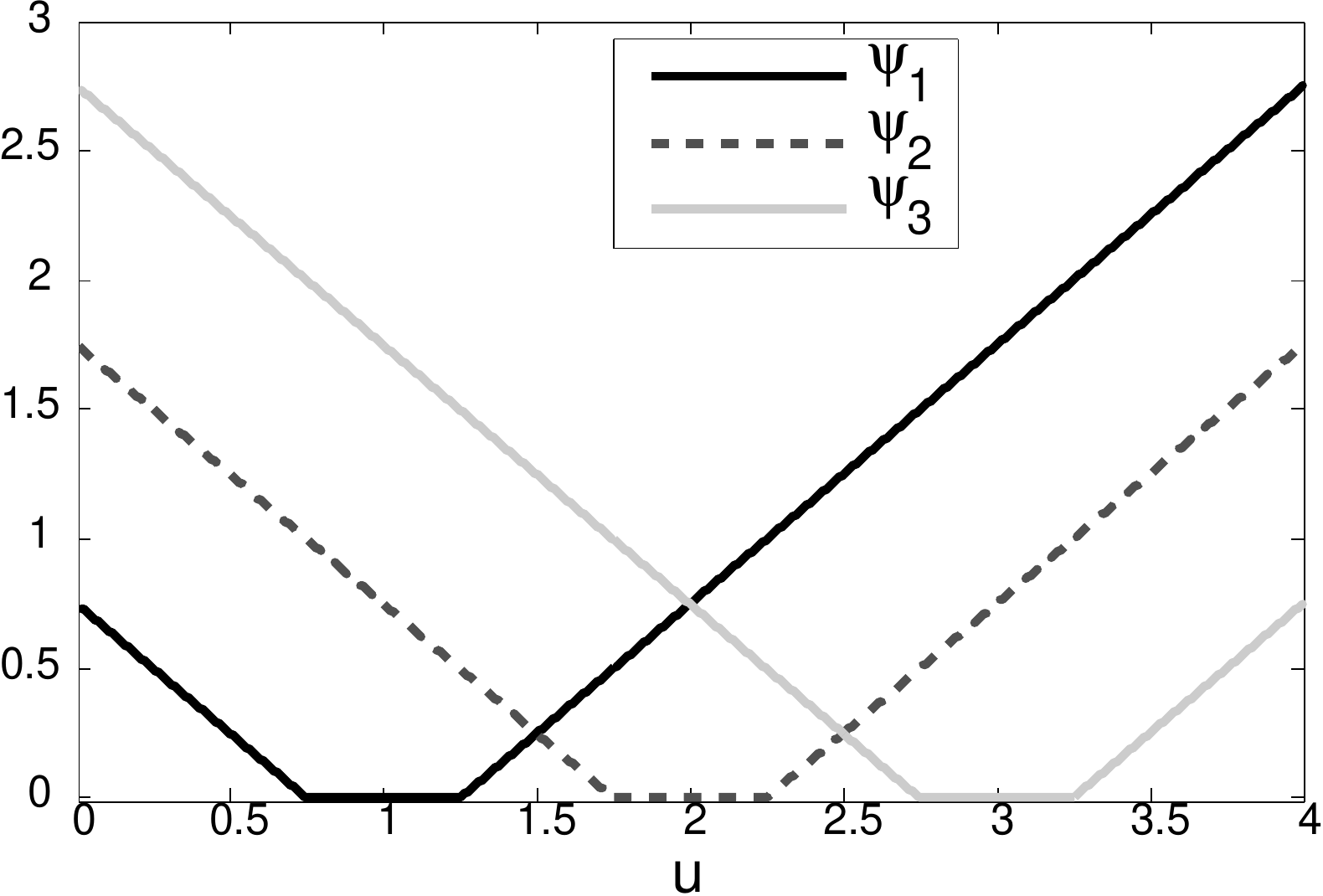}}$
\hspace{1cm}
$\underset{\text{\normalsize{(b)}}}{\includegraphics[width=0.3 \textwidth]{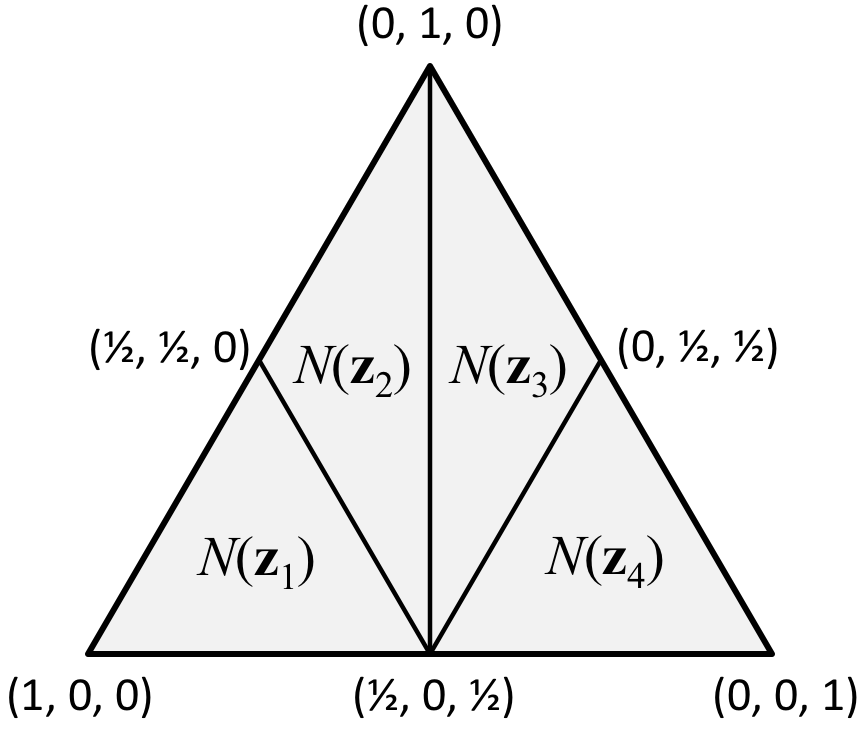}}$
\caption{(a) The `$\epsilon$-insensitive' surrogate $\bpsi^\epsilon:\R\>\R_+^3$ for $\epsilon = 0.25$ (and $n=3$), and (b) its positive normal sets at 4 points $\z_i = \bpsi^{\epsilon}(u_i)\in\R_+^3$ ($i\in[4]$) for $u_1 = 1.25, u_2 = 1.75, u_3 = 2.25, u_4 = 2.75$. See \Exmp{exmp:pos-normals-eps-surrogate} for details.}
\label{fig:eps-surrogate}
\end{center}
\end{figure}

\section{Convex Calibration Dimension}
\label{sec:ccdim}

We now turn to the study of a fundamental quantity associated with the property of $\L$-calibration. Specifically, in Examples~6 and 7 above, we saw that to develop a surrogate calibrated w.r.t.\ to the ordinal regression loss $\L^\ord$ for $n=3$, it was sufficient to consider a surrogate prediction space $\C=\R$, with dimension $d=1$; in addition, the surrogates we considered were convex, and can therefore be used in developing computationally efficient algorithms. In fact the same surrogate prediction space with $d=1$ can be used to develop similar convex surrogate losses calibrated w.r.t.\ the $\L^\ord$ for \emph{any} $n\in\Z_+$. However not all multiclass loss matrices $\L$ have such `low-dimensional' convex surrogates. This raises the natural question of what is the smallest dimension $d$ that supports a convex $\L$-calibrated surrogate for a given multiclass loss $\L$, and leads us to the following definition:

\vspace{4pt}
\begin{defn}[Convex calibration dimension]
Let $\L\in\R_+^{n \times k}$. Define the \emph{convex calibration dimension (CC dimension)} of $\L$ as
\begin{eqnarray*}
\CCdim(\L)
    & \= &
    \min\big\{
        d\in\Z_+: \exists~ \textrm{a convex set $\C\subseteq\R^d$ and a convex surrogate $\bpsi:\C\>\R_+^n$}
    \big.
\\
    & &
    \big.
        ~~~~~~~~~~~\textrm{that is \ $\L$-calibrated}
    \big\}
    \,,
\end{eqnarray*}
if the above set is non-empty, and $\CCdim(\L)=\infty$ otherwise.
\end{defn}

The CC-dimension of a loss matrix $\L$ provides an important measure of the `complexity' of designing convex calibrated surrogates for 
$\L$. Indeed, while the computational complexity of minimizing a surrogate loss, as well as that of converting surrogate predictions into target predictions, can depend on factors other than the dimension $d$ of the surrogate space $\C\subseteq\R^d$, in the absence of other guiding factors, one would in general prefer to use a surrogate in a lower dimension $d$ since this involves learning a smaller number of real-valued functions.

From the above discussion, $\CCdim(\L^\ord) = 1$ for all $n$. In the following, we will be interested in developing an understanding of the CC dimension for general loss matrices $\L$, and in particular in deriving upper and lower bounds on this quantity.

\subsection{Upper Bounds on the Convex Calibration Dimension}
\label{subsec:ccdim-upper-bounds}

We start with a simple result that establishes that the CC dimension of any multiclass loss matrix $\L$ is finite, and in fact is strictly smaller than the number of class labels $n$.
\begin{lem}
\label{lem:surrogate-universal}
Let $\L\in\R_+^{n\times k}$. Let $\C = \big\{ \u\in\R_+^{n-1}: \sum_{j=1}^{n-1} u_j \leq 1 \big\}$, and for each $y\in[n]$, let $\psi_y:\C\>\R_+$ be given by
\[
\psi_y(\u)
    ~ = ~
    \1(y\neq n)\,(u_y - 1)^2 + \sum_{j\in[n-1], j\neq y} u_j{}^2  
    \,.
\]
Then $\bpsi$ is $\L$-calibrated. In particular, since $\bpsi$ is convex, $\CCdim(\L) \leq n-1$.
\end{lem}
It may appear surprising that the convex surrogate $\bpsi$ in the above lemma,
operating on a surrogate space $\C\subset\R^{n-1}$, 
is $\L$-calibrated for \emph{all} multiclass losses $\L$ on $n$ classes. 
However this makes intuitive sense, since in principle, for any multiclass problem, if one can estimate the conditional probabilities of the $n$ classes accurately (which requires estimating $n-1$ real-valued functions on $\X$), then one can predict a target label that minimizes the expected loss according to these probabilities. Minimizing the above surrogate effectively corresponds to such class probability estimation. Indeed, the above lemma can be shown to hold for any surrogate that is a strictly proper composite multiclass loss \citep{Vernet+11}.

In practice, when the number of class labels $n$ is large (such as in a sequence labeling task, where $n$ is exponential in the length of the input sequence), the above result is not very helpful; in such cases, it is of interest to develop algorithms operating on a surrogate prediction space in a lower-dimensional space. Next we give a different upper bound on the CC dimension that depends on the loss $\L$, and for certain losses, can be significantly tighter than the general bound above.

\vspace{4pt}
\begin{thm}
\label{thm:ccdim-upper-bound}
Let $\L\in\R_+^{n\times k}$. Then $\CCdim(\L) \leq \affdim(\L)$.
\end{thm}
\begin{proof}
Let $\affdim(\L) = d$. We will construct a convex $\L$-calibrated surrogate loss $\bpsi$ with surrogate prediction space $\C\subseteq\R^d$.

Let $\V\subseteq\R^n$ denote the ($d$-dimensional) subspace parallel to the affine hull of the column vectors of $\L$, and let $\r\in\R^n$ be the corresponding translation vector, so that $\V = \aff(\{\bell_1,\ldots,\bell_k\}) + \r$. Let $\v_1,\ldots,\v_d \in\V$ be $d$ linearly independent vectors in $\V$.
Let $\{\e_1,\ldots,\e_d\}$ denote the standard basis in $\R^d$, and define a linear function $\tilde{\bpsi}:\R^d\>\R^n$ by 
\[
\tilde{\bpsi}(\e_j) = \v_j 
	~~~~ \forall j \in [d] 
	\,.
\]
Then for each $\v\in\V$, there exists a unique vector $\u\in\R^d$ such that $\tilde{\bpsi}(\u)=\v$. In particular, since $\bell_t + \r \in\V ~\forall t\in[k]$, there exist unique vectors $\u_1,\ldots,\u_k\in\R^d$ such that for each $t\in[k]$, $\tilde{\bpsi}(\u_t) = \bell_t + \r$. Let $\C = \conv(\{\u_1,\ldots,\u_k\}) \subseteq\R^d$, and define $\bpsi:\C\>\R_+^n$ as
\[
\bpsi(\u) = \tilde{\bpsi}(\u) - \r
	~~~~ \forall \u \in \C 
	\,.
\]
To see that $\bpsi(\u) \in \R_+^n ~\forall \u\in\C$, note that for any $\u\in\C$, $\exists \balpha\in\Delta_k$ such that $\u = \sum_{t=1}^k \alpha_t \u_t$, which gives $\bpsi(\u) = \tilde{\bpsi}(\u) - \r = \big( \sum_{t=1}^k \alpha_t \tilde{\bpsi}(\u_t) \big) - \r = \big( \sum_{t=1}^k \alpha_t (\bell_t + \r) \big) - \r = \sum_{t=1}^k \alpha_t \bell_t$ (and $\bell_t\in\R_+^n ~\forall t\in[k]$). 
The function $\bpsi$ is clearly convex. To show $\bpsi$ is $\L$-calibrated, we will use \Thm{thm:condition-sufficient}. Specifically, consider the $k$ points $\z_t = \bpsi(\u_t) = \bell_t \in\cR_\psi$ for $t\in[k]$. 
By definition of $\bpsi$, we have $\S_\psi = \conv(\bpsi(\C)) = \conv(\{\bell_1,\ldots,\bell_k\})$; from the definitions of positive normals and trigger probabilities, it then follows that $\cN^\psi(\z_t)=\cN^\psi(\bell_t)=\Q^\L_t$ for all $t\in[k]$. Thus by Theorem \ref{thm:condition-sufficient}, $\bpsi$ is $\L$-calibrated.
\end{proof}

Since $\affdim(\L)$ is equal to either $\rank(\L)$ or $\rank(\L) - 1$, this immediately gives us the following corollary:

\begin{cor}
\label{cor:ccdim-upper-bound}
Let $\L\in\R_+^{n\times k}$. Then $\CCdim(\L) \leq \rank(\L)$.
\end{cor}
\begin{proof}
Follows immediately from \Thm{thm:ccdim-upper-bound} and the fact that $\affdim(\L) \leq \rank(\L)$.
\end{proof}

\begin{exmp}[CC dimension of Hamming loss]
\label{exmp:ccdim-hamming}
Let $n=2^r$ for some $r\in\Z_+$, and consider the Hamming loss $\L^\Ham\in\R_+^{n\times n}$ defined in \Exmp{exmp:loss-hamming}. As in \Exmp{exmp:loss-hamming}, for each $z\in\{0,\ldots,2^r-1\}$, let $z_i\in\{0,1\}$ denote the $i$-th bit in the $r$-bit binary representation of $z$. 
For each $y\in[n]$, define $\bsigma_y\in\{\pm1\}^r$ as 
\[
\sigma_{yi} 
	~ = ~ 2\, (y-1)_i - 1  
	~ = ~ \begin{cases}
	+1 & ~~\mbox{if $(y-1)_i = 1$} \\
	-1 & ~~\mbox{otherwise.}	
\end{cases}
\]
Then we have 
\begin{eqnarray*}
\ell^\Ham_{yt} 
	& = & 
	\sum_{i=1}^r \1\big( (y-1)_i \neq (t-1)_i \big) 
\\
	& = & 
	\sum_{i=1}^r \bigg( \frac{1-\sigma_{yi}\sigma_{ti}}{2} \bigg)
\\	
	& = & 
	\frac{r}{2} - \sum_{i=1}^r \frac{\sigma_{yi}\sigma_{ti}}{2}
	~~~~\forall y,t\in[n]
	\,.
\end{eqnarray*}
Thus $\affdim(\L^\Ham)\leq r$, and therefore by \Thm{thm:ccdim-upper-bound}, we have $\CCdim(\L^{\Ham})\leq r$. This is a significantly tighter upper bound than the bound of $2^r-1$ given by \Lem{lem:surrogate-universal}.
\end{exmp}

\subsection{Lower Bound on the Convex Calibration Dimension}
\label{subsec:ccdim-lower-bound}

In this section we give a lower bound on the CC dimension of a loss matrix $\L$ and illustrate it by using it to calculate the CC dimension of the 0-1 loss. In \Sec{sec:ranking} we will explore applications of the lower bound to obtaining impossibility results on the existence of convex calibrated surrogates in low-dimensional surrogate spaces for certain types of subset ranking losses. We will need the following definition:

\begin{defn}[Feasible subspace dimension]
The \emph{feasible subspace dimension} of a convex set $\Q\subseteq\R^n$ at a point $\p\in\Q$, denoted by $\mu_\Q(\p)$, is defined as the dimension of the subspace $\F_\Q(\p)\cap(-\F_\Q(\p))$, where $\F_\Q(\p)$ is the cone of feasible directions of $\Q$ at $\p$.\footnote{For a set $\Q\subseteq\R^n$ and point $\p\in\Q$, the cone of feasible directions of $\Q$ at $\p$ is defined as \\ $\F_\Q(\p) = \{\v\in\R^n:~ \exists \epsilon_0>0 ~\textrm{such that}~ \p+\epsilon\v \in \Q ~\forall \epsilon\in(0,\epsilon_0)\}$.} 
\end{defn}

In essence, the feasible subspace dimension of a convex set $Q$ at a point $\p\in\Q$ is simply the dimension of the smallest face of $\Q$ containing $\p$; see \Fig{fig:feasible-subspace-dim} for an illustration. 

\begin{figure}[t]
\begin{center}
\begin{subfigure}[b]{0.3\textwidth}
\includegraphics[width=\textwidth]{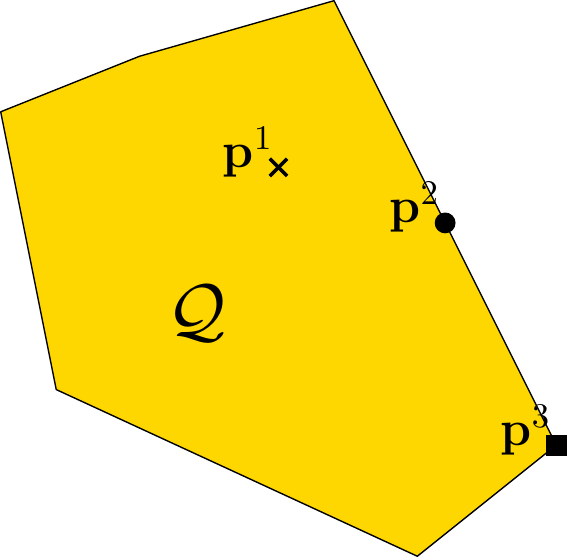}
\caption{Convex set $\Q$ }
\end{subfigure}
\\
\vspace{1em}
\begin{subfigure}[b]{0.31\textwidth}
\includegraphics[width=\textwidth]{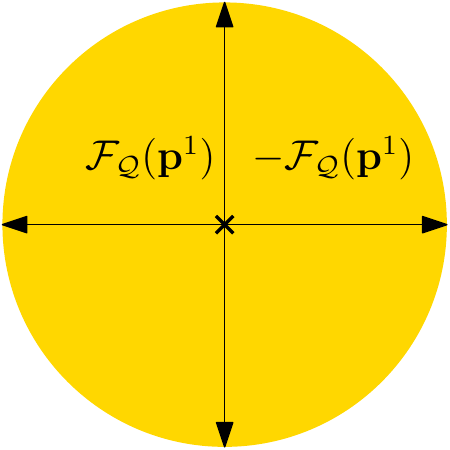}
\caption{\scriptsize $\dim(\F_\Q(\p^1)\cap(-\F_\Q(\p^1)))=2$}
\end{subfigure}
~~
\begin{subfigure}[b]{0.31\textwidth}
\includegraphics[width=\textwidth]{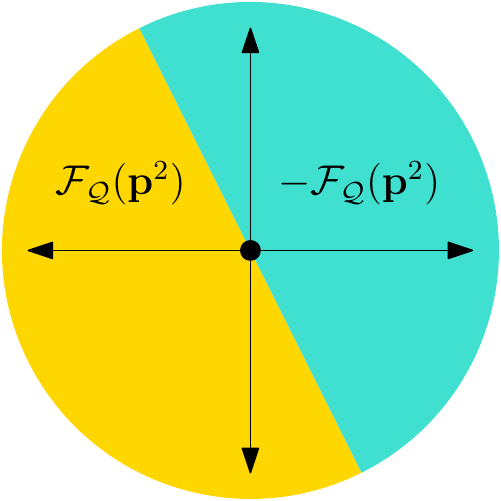}
\caption{\scriptsize $\dim(\F_\Q(\p^2)\cap(-\F_\Q(\p^2)))=1$ }
\end{subfigure}
~~
\begin{subfigure}[b]{0.31\textwidth}
\includegraphics[width=\textwidth]{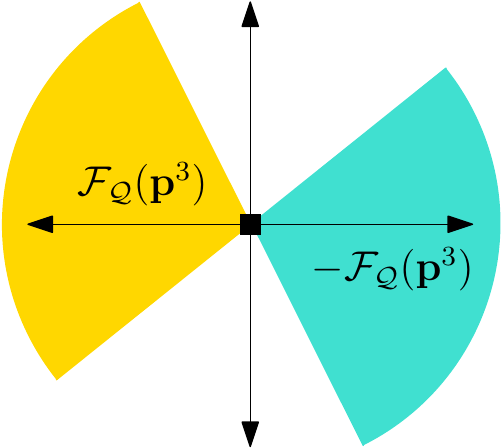}
\caption{\scriptsize $\dim(\F_\Q(\p^3)\cap(-\F_\Q(\p^3)))=0$ }
\end{subfigure}
\end{center}
\caption[Illustration of feasible subspace dimension $\mu_\Q(\p)$]{Illustration of feasible subspace dimension $\mu_\Q(\p)$ of a 2-dimensional convex set $\Q$ at three points $\p=\p^1,\p^2,\p^3$. Here $\mu_\Q(\p^1)=2$, $\mu_\Q(\p^2)=1$, and $\mu_\Q(\p^3)=0$.}
\label{fig:feasible-subspace-dim}
\end{figure}

Both the proof of the lower bound we will provide below and its applications make use of the following lemma, which gives a method to calculate the feasible subspace dimension for certain convex sets $\Q$ and points $\p\in\Q$:

\begin{lem}
\label{lem:feasible-subspace-dimension}
Let $\Q = \big\{ \q\in\R^n: \A^1 \q \leq \b^1, \A^2 \q \leq \b^2, \A^3 \q = \b^3 \big\}$. Let $\p\in\Q$ be such that $\A^1\p=\b^1$, $\A^2\p<\b^2$. Then $\mu_\Q(\p) = \nullity\Big( \Big[ \begin{matrix} \A^1 \\ \A^3 \end{matrix} \Big] \Big)$.
\end{lem}
\begin{proof} 
We will show that $\F_\Q(\p)\cap(-\F_\Q(\p)) = \Null\Big( \Big[ \begin{matrix} \A^1 \\ \A^3 \end{matrix} \Big] \Big)$, from which the lemma follows.
\\[6pt]
First, let $\v \in \Null\Big( \Big[ \begin{matrix} \A^1 \\ \A^3 \end{matrix} \Big] \Big)$. Then for $\epsilon > 0$, we have
\begin{eqnarray*}
 \A^1 (\p+\epsilon\v) 
 	& = &
	\A^1\p + \epsilon \A^1\v  ~ = ~  \A^1\p + \0 ~ = ~ \b^1
\\
\A^2(\p+\epsilon \v)
	& < &
	\b^2 ~~~ \text{for small enough $\epsilon$,  since $\A^2\p<\b^2$}
\\
\A^3 (\p+\epsilon\v)
	& = &
	\A^3\p +  \epsilon \A^3\v  ~ = ~  \A^3\p + \0 ~ = ~ \b^3
	\,.
\end{eqnarray*}
Thus $\v\in\F_\Q(\p)$. Similarly, we can show $-\v\in\F_\Q(\p)$. Thus $\v\in \F_\Q(\p) \cap (-\F_\Q(\p))$, giving $\Null\Big( \Big[ \begin{matrix} \A^1 \\ \A^3 \end{matrix} \Big] \Big) \subseteq \F_\Q(\p)\cap(-\F_\Q(\p))$.
Now let $\v\in \F_\Q(\p)\cap(-\F_\Q(\p))$. Then for small enough $\epsilon>0$, we have both $\A^1(\p+\epsilon \v)\leq \b^1$ and $\A^1(\p-\epsilon \v)\leq \b^1$. Since $\A^1\p=\b^1$, this gives $\A^1 \v =\0$. Similarly, for small enough $\epsilon>0$, we have $\A^3(\p+\epsilon \v) = \b^3$; since $\A^3\p = \b^3$, this gives $\A^3 \v = \0$. Thus $\Big[ \begin{matrix} \A^1 \\ \A^3 \end{matrix} \Big] \v = \0$, giving $\F_\Q(\p)\cap(-\F_\Q(\p))\subseteq \Null\Big( \Big[ \begin{matrix} \A^1 \\ \A^3 \end{matrix} \Big] \Big)$.
\end{proof}

The following gives a lower bound on the CC dimension of a loss matrix $\L$ in terms of the feasible subspace dimension of the trigger probability sets $\Q^{\L}_t$ at points $\p\in\Q^\L_t$:

\begin{thm}
\label{thm:ccdim-lower-bound}
Let $\L\in\R_+^{n\times k}$. Let $\p\in \Delta_n$ and $t\in\arg\min_{t'} \p^\top \bell_{t'}$ (equivalently, let $\p\in\Q^\L_t$). 
Then 
\[
\CCdim(\L)
    \geq
    \| \p \|_0 - \mu_{\Q^{\L}_t}(\p) - 1
    \,.
\]
\end{thm}


The above lower bound allows us to calculate precisely the CC dimension of the 0-1 loss:

\begin{exmp}[CC dimension of 0-1 loss]
\label{exmp:ccdim-0-1}
Let $n\in\Z_+$, and consider the 0-1 loss $\L^\zo\in\R_+^{n\times n}$ defined in Example~1. Take $\p = (\frac{1}{n},\ldots,\frac{1}{n})^\top \in \Delta_n$. Then $\p\in\Q^\zo_t$ for all $t\in[k]=[n]$ (see \Fig{fig:trigger-prob}); in particular, we have $\p\in\Q^\zo_1$. Now $\Q^\zo_1$ can be written as 
\begin{eqnarray*}
\Q^\zo_1
	& = &
	\big\{ \q\in\Delta_n: q_1 \geq q_y ~\forall y\in\{2,\ldots,n\} \big\}
\\
	& = &
	\big\{ \q\in\R^n: \big[ {-\e_{n-1}} \,\,\, \I_{n-1} \big] \q \leq \0, -\q \leq \0, \e_n^\top \q = 1 \}
	\,, 
\end{eqnarray*}
where $\e_{n-1}$, $\e_n$ denote the $(n-1)\times 1$ and $n\times 1$ all ones vectors, respectively, and $\I_{n-1}$ denotes the $(n-1)\times(n-1)$ identity matrix.
Moreover, we have $\big[ {-\e_{n-1}} \,\,\, \I_{n-1} \big] \p = 0$, $-\p < \0$. Therefore, by \Lem{lem:feasible-subspace-dimension}, we have 
\[
\mu_{\Q^\zo_1}(\p) 
	= 
	\nullity\left( \left[ \begin{matrix}  
		{-\e_{n-1}} \,\,\, \I_{n-1}  \\[2pt] 
		\e_n^\top 
	\end{matrix} \right] \right) 
	= 
	\nullity\left( \left[ \begin{matrix} 
		-1 \,\, 1 \,\, 0 \,\, \ldots \,\, 0 \\
		-1 \,\, 0 \,\, 1 \,\, \ldots \,\, 0 \\
		\vdots \\
		-1 \,\, 0 \,\, 0 \,\, \ldots \,\, 1 \\
		 ~~~ 1 \,\, 1 \,\, 1 \,\, \ldots \,\, 1 
	\end{matrix} \right] \right)
	= 
	0 
	\,.
\]
Moreover, $\|\p\|_0 = n$. Thus by \Thm{thm:ccdim-lower-bound}, we have $\CCdim(\L^\zo) \geq n-1$. Combined with the upper bound of \Lem{lem:surrogate-universal}, this gives $\CCdim(\L^\zo) = n-1$.
\end{exmp}

\subsection{Tightness of Bounds}
\label{subsec:tight}

The upper and lower bounds above are not necessarily tight in general. For example, for the $n$-class ordinal regression loss of Example~2, we know that $\CCdim(\L^\ord) = 1$; however the upper bound of \Thm{thm:ccdim-upper-bound} only gives $\CCdim(\L^\ord) \leq n-1$. Similarly, for the $n$-class abstain loss of Example~4, it can be shown that $\CCdim(\L^\abstain) = O(\ln n)$ (in fact we conjecture it to be $\Theta(\ln n)$) \citep{Ramaswamy+15}, whereas the upper bound of \Thm{thm:ccdim-upper-bound} gives $\CCdim(\L^\abstain) \leq n$, and the lower bound of \Thm{thm:ccdim-lower-bound} yields only $\CCdim(\L^\abstain) \geq 1$. However, as we show below, for certain losses $\L$, the bounds of Theorems \ref{thm:ccdim-upper-bound} and \ref{thm:ccdim-lower-bound} are in fact tight (upto an additive constant of $1$).

\begin{lem}
\label{lem:tight}
Let $\L\in\R_+^{n \times k}$. Let $\p\in\relint(\Delta_n)$ and $c\in\R_+$ be such that $\p^\top \bell_t = c ~\forall t\in[k]$.
Then $\forall t\in[k]$, 
\[
\mu_{\Q^\L_t}(\p) \leq n - \affdim(\L)
	\,.
\]
\end{lem}
\begin{proof}
Since $\p^\top \bell_t = c ~\forall t\in[k]$, we have $\p\in\Q^\L_t ~\forall t\in[k]$. In particular, we have $\p\in\Q^\L_1$.
Now 
\[
\Q^\L_1 
	~ = ~
\left\{ \q\in\R^n: 
	\left[ \begin{matrix}
		(\bell_2-\bell_1)^\top \\
		\vdots \\
		(\bell_k-\bell_1)^\top
	\end{matrix} \right] \q \geq \0, 
	-\q \leq \0, 
	\e_n^\top \q=1 
\right\}
	\,.
\]
Moreover, $\left[ \begin{matrix} (\bell_2-\bell_1)^\top \\ \vdots \\ (\bell_k-\bell_1)^\top \end{matrix} \right] \p = \0$ and $-\p < \0$. Therefore, by \Lem{lem:feasible-subspace-dimension}, we have 
\[
\mu_{\Q^\L_1}(\p) 
	~ = ~ 
	\nullity\left( \left[ \begin{matrix}
		(\bell_2 - \bell_1)^\top \\ 
		\vdots \\
		(\bell_k - \bell_1)^\top \\
		\e_n^\top
	\end{matrix}
	\right] \right)
	~ = ~
	n - \rank\left( \left[ \begin{matrix}
		(\bell_2 - \bell_1)^\top \\ 
		\vdots \\
		(\bell_k - \bell_1)^\top \\
		\e_n^\top
	\end{matrix}
	\right] \right)
	~ \leq ~ 
	n - \affdim(\L)
	\,.
\]
A similar proof holds for $\mu_{\Q^\L_t}(\p)$ for all other $t\in[k]$.
\end{proof}

Combining the above result with \Thm{thm:ccdim-lower-bound} immediately gives the following:

\begin{thm}
\label{thm:tight}
Let $\L\in\R_+^{n \times k}$. If $\exists \p\in\relint(\Delta_n), c\in\R_+$ such that $\p^\top \bell_t = c ~\forall t\in[k]$,
then 
\[
\CCdim(\L) \geq \affdim(\L) - 1
	\,.
\]
\end{thm}
\begin{proof}
Follows immediately from \Thm{thm:ccdim-lower-bound} and \Lem{lem:tight}.
\end{proof}


Intuitively, the condition that $\exists \p\in\relint(\Delta_n), c\in\R_+$ such that $\p^\top \bell_t = c ~\forall t\in[k]$ in \Lem{lem:tight} and \Thm{thm:tight} above captures the essence of a hard problem: in this case, if the underlying label probability distribution $\p'$ is very close to $\p$, then it becomes hard to decide which element $t\in[k]$ is an optimal prediction.
This is essentially what leads the lower bound on the CC-dimension to become tight in this case.

A particularly useful application of \Thm{thm:tight} is to losses $\L$ whose columns $\bell_t$ can be obtained from one another by permuting entries:

\begin{cor}
\label{cor:col-perm}
Let $\L\in\R_+^{n \times k}$ be such that all columns of $\L$ can be obtained from one another by permuting  entries, i.e.\ $\forall t_1, t_2 \in[k], ~\exists \sigma\in\Pi_n$ such that $\ell_{y,t_2} = \ell_{\sigma(y),t_1} ~\forall y\in[n]$. Then 
\[
\CCdim(\L) \geq \affdim(\L) - 1
	\,.
\]
\end{cor}
\begin{proof}
Let $\p=\big( \frac{1}{n}, \ldots, \frac{1}{n} \big)^\top \in \relint(\Delta_n)$. Let $c = \frac{\| \bell_1\|_1}{n}$. Then under the given condition, $\p^\top\bell_t = c ~\forall t\in[k]$. The result then follows from \Thm{thm:tight}.
\end{proof}

We will use the above corollary in establishing lower bounds on the CC dimension of certain subset ranking losses below.

\section{Applications to Subset Ranking}
\label{sec:ranking}

We now consider applications of the above framework to analyzing various subset ranking problems, where each instance $x\in\X$ consists of a query together with a set of $r$ documents (for simplicity, $r\in\Z_+$ here is fixed), and the goal is to learn a prediction model which given such an instance predicts a ranking (permutation) of the $r$ documents \citep{CossockZh08}.\footnote{The term `subset ranking' here refers to the fact that in a query-based setting, each instance involves a different `subset' of documents to be ranked; see \citep{CossockZh08}.}
We consider three popular losses used for subset ranking: the normalized discounted cumulative gain (NDCG) loss, the pairwise disagreement (PD) loss, and the mean average precision (MAP) loss.\footnote{Note that NDCG and MAP are generally expressed as gains, where a higher value corresponds to better performance; we can express them as non-negative losses by subtracting them from a suitable constant.} Each of these subset ranking losses can be viewed as a specific type of multiclass loss acting on a certain label space $\Y$ and prediction space $\hat{\Y}$. In particular, for the NDCG loss, the label space $\Y$ contains $r$-dimensional multi-valued relevance vectors; for PD loss, $\Y$ contains directed acyclic graphs on $r$ nodes; and for MAP loss, $\Y$ contains $r$-dimensional binary relevance vectors. In each case, the prediction space $\hat{\Y}$ is the set of permutations of $r$ objects: $\hat{\Y} = \Pi_r$. We study the convex calibration dimension of these losses below. 
Specifically, we show that the CC dimension of the NDCG loss is upper bounded by $r$ (\Sec{subsec:NDCG}), and that 
of both the PD and MAP losses is lower bounded by a quadratic function of $r$ (Sections~\ref{subsec:PD} and \ref{subsec:MAP}).
Our result on the CC dimension of the NDCG  loss is consistent with previous results in the literature showing the existence of $r$-dimensional convex calibrated surrogates for NDCG \citep{Ravikumar+11, Buffoni+11};
our results on the CC dimension of the PD and MAP losses strengthen previous results of \cite{Calauzenes+12}, who showed non-existence of $r$-dimensional convex calibrated surrogates (with a fixed argsort predictor) for PD and MAP.

\subsection{Normalized Discounted Cumulative Gain (NDCG)}
\label{subsec:NDCG}

The NDCG loss is widely used in information retrieval applications \citep{JarvelinKe00}. Here $\Y$ is the set of $r$-dimensional relevance vectors with say $s$ relevance levels, $\Y = \{0,1,\ldots,s-1\}^r$, and $\hat{\Y}$ is the set of permutations of $r$ objects, $\hat{\Y}=\Pi_r$ (thus here $n =|\Y|=s^r$ and $k = |\hat{\Y}|=r!$). 
The loss on predicting a permutation $\sigma\in\Pi_r$ when the true label is $\y\in\{0,1,\ldots,s-1\}^r$ is given by
\[
\ell^\NDCG(\y, \sigma)
	~ = ~
	1-\frac{1}{z(\y)} \sum_{i=1}^r \frac{2^{y_i}-1}{\log_2(\sigma(i)+1)}
	\,,
\]
where $z(\y)$ is a normalizer that ensures the loss is non-negative and depends only on $\y$.
The NDCG loss can therefore be viewed as a multiclass loss matrix $\L^{\NDCG}\in\R_+^{s^r\times r!}$. 
Clearly, $\affdim(\L^\NDCG) \leq r$, and therefore by \Thm{thm:ccdim-upper-bound}, we have 
\[
\CCdim(\L^\NDCG) ~ \leq ~ r
	\,.
\]
Indeed, previous results in the literature have shown the existence of $r$-dimensional convex calibrated surrogates for NDCG \citep{Ravikumar+11, Buffoni+11}.

\subsection{Pairwise Disagreement (PD)}
\label{subsec:PD}

Here the label space $\Y$ is the set of all directed acyclic graphs (DAGs) on $r$ vertices, which we shall denote as $\G_r$; for each directed edge $(i,j)$ in a graph $G\in\G_r$ associated with an instance $x\in\X$, the $i$-th document in the document set in $x$ is preferred over the $j$-th document. The prediction space $\hat{\Y}$ is again the set of permutations of $r$ objects, $\hat{\Y}=\Pi_r$. The loss on predicting a permutation $\sigma\in\Pi_r$ when the true label is $G\in\G_r$ is given by
\begin{eqnarray*}
\ell^\pair(G, \sigma)
	& = &
	\sum_{(i,j)\in G}
		\1\big( \sigma(i) > \sigma(j) \big) 
\\
	& = & 
	\sum_{i=1}^r \sum_{j=1}^r 
		\1\big( (i,j)\in G \big) \cdot \1\big( \sigma(i)> \sigma(j) \big) 
\\
	& = &
	\sum_{i=1}^r \sum_{j=1}^{i-1} 
		\1\big( (i,j)\in G \big) \cdot \1\big( \sigma(i)> \sigma(j) \big) + 
		\1\big( (j,i)\in G \big) \cdot \Big( 1-\1\big(\sigma(i)> \sigma(j) \big) \Big) 
\\
	& = & 
	\sum_{i=1}^r \sum_{j=1}^{i-1} 
		\Big( \1\big( (i,j)\in G \big) - \1\big( (j,i)\in G \big) \Big) \cdot \1\big( \sigma(i)> \sigma(j) \big) + 
	\sum_{i=1}^r \sum_{j=1}^{i-1} 
		\1\big( (j,i)\in G \big)
	\,.
\end{eqnarray*}
The PD loss can be viewed as a multiclass loss matrix $\L^\pair\in\R_+^{|\G_r|\times r!}$.
Note that the second term in the sum above depends only the label $G$; removing this term amounts to simply subtracting a fixed vector from each column of the loss matrix, which does not change the properties of the minimizer of the loss or its CC dimension. We can therefore consider the following loss instead:
\[
\tilde{\ell}^\pair(G,\sigma)
	~ = ~
	\sum_{i=1}^r \sum_{j=1}^{i-1} 
		\Big( \1\big( (i,j)\in G \big) - \1\big( (j,i)\in G \big) \Big) \cdot \1\big( \sigma(i)> \sigma(j) \big) 
	\,.
\]
The resulting loss matrix $\tilde{\L}^\pair$ clearly has rank at most $\frac{r(r-1)}{2}$. Therefore, by \Cor{cor:ccdim-upper-bound}, we have 
\[
\CCdim(\L^\pair) 
	~ = ~ 
	\CCdim(\tilde{\L}^\pair) 
	~ \leq ~
	\frac{r(r-1)}{2} 
	\,.
\] 
In fact one can show that the rank of $\tilde{\L}^\pair$ is exactly $\frac{r(r-1)}{2}$:

\begin{prop}
\label{prop:PD-rank}
$\rank(\tilde{\L}^\pair) = \frac{r(r-1)}{2}$.
\end{prop}

Moreover, it is easy to see that the columns of $\tilde{\L}^\pair$ can all be obtained from one another by permuting entries. Therefore, by \Cor{cor:col-perm}, we also have 
\[
\CCdim(\L^\pair) 
	~ = ~ 
	\CCdim(\tilde{\L}^\pair) 
	~ \geq ~
	\frac{r(r-1)}{2} - 2
	\,.
\]
Informally, this implies that a convex surrogate that achieves calibration w.r.t.\ $\L^\pair$ over the full probability simplex must effectively `estimate' all edge weights.
Formally, this strengthens previous results of \cite{Duchi+10} and \cite{Calauzenes+12}. In particular, \cite{Duchi+10} showed that certain popular $r$-dimensional convex surrogates are not calibrated for the PD loss, and conjectured that such convex calibrated surrogates (in $r$ dimensions) do not exist; \cite{Calauzenes+12} showed that indeed there do not exist any $r$-dimensional convex surrogates that use argsort as the predictor and are calibrated for the PD loss. The above result allows us to go further and conclude that in fact, one cannot design convex calibrated surrogates for the PD loss in any prediction space of less than $\frac{r(r-1)}{2} - 2$ dimensions (regardless of the predictor used). 

\subsection{Mean Average Precision (MAP)}
\label{subsec:MAP}

Here the label space $\Y$ is the set of all (non-zero) $r$-dimensional binary relevance vectors, $\Y = \{0,1\}^r\setminus \{\0\}$, and the prediction space $\hat{\Y}$ is again the set of permutations of $r$ objects, $\hat{\Y}=\Pi_r$. The loss on predicting a permutation $\sigma\in\Pi_r$ when the true label is $\y\in\{0,1\}^r\setminus\{\0\}$ is given by
\begin{eqnarray}
\ell^\MAP(\y,\sigma)
	& = &
	1 - \frac{1}{\|\y\|_1} \sum_{i:y_i=1} \frac{1}{\sigma(i)}\sum_{j=1}^{\sigma(i)} y_{\sigma^{-1}(j)} 
\nonumber \\
	& = & 
	1 - \frac{1}{\|\y\|_1} \sum_{i=1}^r \frac{y_i}{\sigma(i)} \sum_{j=1}^{\sigma(i)}  y_{\sigma^{-1}(j)}
\nonumber \\
	& = & 
	1 - \frac{1}{\|\y\|_1} \sum_{i=1}^r  \sum_{j=1}^{i}  \frac{y_{\sigma^{-1}(i)} \, y_{\sigma^{-1}(j)}}{i} 
\nonumber \\
	& = & 
	1 - \frac{1}{\|\y\|_1} \sum_{i=1}^r \sum_{j=1}^i \frac{y_i \, y_j }{\max(\sigma(i),\sigma(j))}
\label{eqn:MAP-low-rank}	
\end{eqnarray}
Thus the MAP loss can be viewed as a multiclass loss matrix $\L^\MAP\in\R_+^{(2^r-1)\times r!}$. Clearly, $\affdim(\L^\MAP)\leq \frac{r(r+1)}{2}$, and therefore by \Thm{thm:ccdim-upper-bound}, we have 
\[
\CCdim(\L^\MAP) 
	~ \leq ~
	\frac{r(r+1)}{2}
	\,.
\]
One can also show the following lower bound on the rank of $\L^\MAP$:

\begin{prop}
\label{prop:MAP-rank-lower-bound}
$\rank(\L^\MAP)\geq \frac{r(r-1)}{2}-2$.
\end{prop}

Again, it is easy to see that the columns of $\L^\MAP$ can all be obtained from one another by permuting entries, and therefore by \Cor{cor:col-perm}, we have 
\[
\CCdim(\L^\MAP) 
	~ \geq ~
	\frac{r(r-1)}{2} - 4
	\,.
\]
This again strengthens a previous result of \cite{Calauzenes+12}, who showed that there do not exist any $r$-dimensional convex surrogates that use argsort as the predictor and are calibrated for the MAP loss. As with the PD loss, the above result allows us to go further and conclude that in fact, one cannot design convex calibrated surrogates for the MAP loss in any prediction space of less than $\frac{r(r-1)}{2} - 4$ dimensions (regardless of the predictor used).

\section{Conclusion}
\label{sec:concl}

We have developed a unified framework for studying consistency properties of surrogate risk minimization algorithms for general multiclass learning problems, defined by a general multiclass loss matrix. In particular, we have introduced the notion of \emph{convex calibration dimension} (CC dimension) of a multiclass loss matrix, a fundamental quantity that measures the smallest `size' of a prediction space in which it is possible to design a convex surrogate that is calibrated with respect to the given loss matrix, 
and have used this to analyze consistency properties of surrogate losses for various multiclass learning problems. 

Our study both generalizes previous results and sheds new light on various multiclass losses. For example, our analysis shows that for the $n$-class 0-1 loss, any convex calibrated surrogate must necessarily entail learning at least $n-1$ real-valued functions, thus showing that the calibrated multiclass surrogate of \cite{Lee+04}, whose minimization entails learning $n$ real-valued functions, is essentially not improvable (in the sense of the number of real-valued functions that need to be learned). Another implication of our study is to the pairwise disagreement (PD) and mean average precision (MAP) losses for subset ranking: while previous results have shown that for subset ranking problems with $r$ documents per query, there do not exist $r$-dimensional convex calibrated surrogates for the PD and MAP losses, our analysis shows that (a) these losses do admit convex calibrated surrogates in higher dimensions, and (b) to obtain such convex calibrated surrogates for these losses, one needs to operate in an $\Omega(r^2)$-dimensional surrogate prediction space (i.e.\ one needs to learn $\Omega(r^2)$ real-valued functions, rather than just $r$ real-valued `scoring' functions).

As discussed in \Sec{subsec:tight}, while the upper and lower bounds we have obtained on the CC dimension are tight (up to an additive constant of 1) for certain classes of loss matrices, they can be quite loose in general. An important open direction is to obtain a characterization of the CC dimension in more general settings.
It would also be useful to develop methods for deriving explicit surrogate regret bounds for general calibrated surrogates, through which one can relate the excess target risk to the excess surrogate risk for any multiclass loss and corresponding calibrated surrogate.
Finally, another interesting direction would be to develop a generic procedure for designing convex calibrated surrogates operating on a `minimal' space according to the CC dimension of a given loss matrix. 
There has been some recent progress in this direction in \citep{Ramaswamy+13}, where a general method is described for designing convex calibrated surrogates in a surrogate space with dimension at most the rank of the given loss matrix. However, while the rank forms an upper bound on the CC dimension of the loss matrix, as discussed above, this bound is not always tight, giving rise to the possibility of designing convex calibrated surrogates in lower-dimensional spaces for certain losses. 
Resolving these issues will contribute significantly to our understanding of the conditions under which convex calibrated surrogates can be designed for a given multiclass learning problem.

\section{Proofs}
\label{sec:proofs}

\subsection{Proof of \Thm{thm:condition-sufficient}}

The proof uses the following technical lemma:
\begin{lem}
\label{lem:condition-sufficient}
Let $\L\in\R_+^{n\times k}$ and $\bpsi:\C\>\R_+^n$. Suppose there exist $r\in\N$ and $\z_1,\ldots,\z_r \in \cR_\psi$ such that $\bigcup_{j=1}^r \cN^{\psi}(\z_j) = \Delta_n$ and for each $j\in[r]$, $\exists t\in[k]$ such that $\cN^{\psi}(\z_j) \subseteq \Q^\L_t$. Then any element $\z\in\S_\psi$ can be written as $\z=\z'+\z''$ for some $\z'\in\conv(\{\z_1,\ldots,\z_r\})$ and $\z''\in\R_+^n$.
\end{lem}
\begin{proof} (Proof of \Lem{lem:condition-sufficient}) \\[6pt]
Let $\S' = \{\z'+\z'': \z'\in\conv(\{\z_1,\ldots,\z_r\}), \z''\in \R_+^n\}$, and suppose there exists a point $\z\in\S_\psi$ which cannot be decomposed as claimed, i.e. such that $\z\notin\S'$. Then by the Hahn-Banach theorem (e.g.\ see \cite{Gallier09}, corollary 3.10), there exists a hyperplane that strictly separates $\z$ from $\S'$, i.e.\ $\exists \w\in\R^n$ such that $\w^\top \z < \w^\top \a ~\forall \a\in\S'$. It is easy to see that $\w\in\R_+^n$ (since a negative component in $\w$ would allow us to choose an element $\a$ from $\S'$ with arbitrarily small $\w^\top \a$).

Now consider the vector $\q=\w/\sum_{i=1}^n w_i \in\Delta_n$. Since $\bigcup_{j=1}^r \cN^{\psi}(\z_j) = \Delta_n$, $\exists j\in[r]$ such that $\q\in \cN^{\psi}(\z_j)$. By definition of positive normals, this gives $\q^\top \z_j \leq \q^\top \z$, and therefore $\w^\top \z_j \leq \w^\top \z$. But this contradicts our construction of $\w$ (since $\z_j\in\S'$). Thus it must be the case that every $\z\in\S_\psi$ is also an element of $\S'$.
\end{proof}

\begin{proof} (Proof of \Thm{thm:condition-sufficient})\\[6pt]
We will show $\L$-calibration of $\bpsi$ via \Lem{lem:pred-pred'}.
For each $j\in[r]$, let
\[
T_j = \Big\{ t\in[k]: \cN^{\psi}(\z_j) \subseteq \Q^\L_t \Big\}
    \,;
\]
by assumption, $T_j \neq \emptyset ~\forall j\in[r]$.
By \Lem{lem:condition-sufficient}, for every $\z\in\S_\psi$, 
$\exists \balpha\in\Delta_{r}, \u \in \R_+^n$ such that 
\(
\z = \sum_{j=1}^{r} \alpha_j \z_j + \u
	\,.
\)
For each $\z\in\S_\psi$, arbitrarily fix a unique $\balpha^\z \in\Delta_r$ and $\u^\z\in\R_+^n$ satisfying the above, i.e.\ such that 
\[
\z=\sum_{j=1}^{r} \alpha_j^\z \z_j + \u^\z \,.
\]
Now define $\pred':\S_\psi\>[k]$ as
\[
    \pred'(\z) = \min\big\{ t\in[k]: \exists j\in[r] ~\mbox{such that $\alpha^\z_j \geq \frac{1}{r}$ and  $t\in T_j$} \big\}
    \,.
\]
We will show $\pred'$ satisfies the condition for $\ell$-calibration.

Fix any $\p\in\Delta_n$. Let
\[
J_\p = \Big\{ j\in[r]: \p\in\cN^{\psi}(\z_j) \Big\}
    \,;
\]
since $\Delta_n = \bigcup_{j=1}^r \cN^{\psi}(\z_j)$, we have $J_\p \neq \emptyset$. Clearly,
\begin{equation}
\forall j\in J_\p:
    \p^\top\z_j = \inf_{\z\in\S_\psi} \p^\top\z
\label{eqn:j-in-Jp}
\end{equation}
\begin{equation}
\forall j\notin J_\p:
    \p^\top\z_j > \inf_{\z\in\S_\psi} \p^\top\z
\label{eqn:j-notin-Jp}
\end{equation}
Moreover, from definition of $T_j$, we have
\[
\forall j\in J_\p: ~~~
    t\in T_j
    \implies
    \p \in \Q^\L_t
    \implies
    t \in \argmin_{t'} \p^\top \bell_{t'}
    \,.
\]
Thus we get
\begin{equation}
\forall j\in J_\p: ~~~
    T_j
    \subseteq
    \argmin_{t'} \p^\top \bell_{t'}
    \,.
\label{eqn:T-j-argmin}
\end{equation}
Now, for any $\z\in\S_\psi$ for which $\pred'(\z) \notin \arg\min_{t'} \p^\top \bell_{t'}$, we must have $\alpha^\z_j \geq \frac{1}{r}$ for at least one $j\notin J_\p$ (otherwise, we would have $\pred'(\z)\in T_j$ for some $j\in J_\p$, giving $\pred'(\z) \in \arg\min_{t'} \p^\top \bell_{t'}$, a contradiction).
Thus we have
\begin{eqnarray}
\inf_{\z\in\S_\psi:\pred'(\z)\notin \argmin_{t'}\p^\top\bell_{t'}} 
	\p^\top \z 
	& = & 
	\inf_{\z\in\S_\psi:\pred'(\z)\notin \argmin_{t'}\p^\top\bell_{t'}} 
		\sum_{j=1}^r\alpha^\z_j\p^\top\z_j + \p^\top \u^\z 
\\
	& \geq & 
	\inf_{\balpha\in\Delta_r:\alpha_j\geq\frac{1}{r} \text{ for some } j\notin J_\p}
		\sum_{j=1}^r\alpha_j\p^\top\z_j
\\
	& \geq & 
	\min_{j\notin J_\p} 
		\inf_{\alpha_j\in[\frac{1}{r},1]} \alpha_j \p^\top \z_j + 
		(1-\alpha_j) \inf_{\z\in\S_\psi}\p^\top \z  
\\									 
	& > & 
	\inf_{\z\in\S_\psi}\p^\top \z
	\,,
\end{eqnarray}
where the last inequality follows from \Eqn{eqn:j-notin-Jp}.
Since the above holds for all $\p\in\Delta_n$, by \Lem{lem:pred-pred'}, we have that $\bpsi$ is $\L$-calibrated.
\end{proof}


\subsection{Proof of \Lem{lem:surrogate-universal}}

\begin{proof}
For each $\u \in \C$, 
define $\p^{\u} = \bigg( \begin{matrix} \u \\ 1-\sum_{j=1}^{n-1} u_j \end{matrix} \bigg) \in \Delta_n$. Define $\pred:\C\>[k]$ as 
\[
\pred(\u) ~ = ~ \min\big\{ t\in[k]:  \p^{\u}   \in Q^\L_t \big\}
	\,.
\]
We will show that $\pred$ satisfies the condition of Definition 1.

Fix $\p \in \Delta_n$. It can be seen that 
\[
\p^\top \bpsi(\u) 
	~ = ~ 
	\sum_{j=1}^{n-1} \Big( p_j (u_j -1)^2 + (1-p_j) \, u_j{}^2 \Big)
	\,.
\]
Minimizing the above over $\u$ yields the unique minimizer $\u^* = (p_1, \ldots, p_{n-1})^\top
\in \C$, which after some calculation gives
\[
\inf_{\u \in\C} \p^\top \bpsi(\u) 
	~ = ~
	\p^\top \bpsi(\u^*) 
	~ = ~
	\sum_{j=1}^{n-1} p_j(1-p_j) 
	\,.
\]
Now, for each $t\in[k]$, define 
\[
\reg^\L_\p(t) 
	~ \= ~
	\p^\top \bell_t - \min_{t'\in[k]} \p^\top \bell_{t'}
	\,.
\]
Clearly, $\reg^\L_\p(t) = 0 \Longleftrightarrow \p\in\Q^\L_t$. Note also that $\p^{\u^*} = \p$, and therefore $\reg^\L_\p(\pred(\u^*)) = 0$. Let 
\[
\epsilon 
	~ = ~ 
	\min_{t\in[k]: \p\notin \Q^\L_t} \reg^\L_\p(t) 
	~ > ~
	0
	\,.
\]
Then we have
\begin{eqnarray}
\inf_{\u \in\C:\pred(\u)\notin \argmin_t{\p^\top\bell_t} } \p^\top \bpsi(\u) 
	& = &
	\inf_{\u\in\C:\reg^\L_\p(\pred(\u)) \geq \epsilon} \p^\top \bpsi(\u) 
\\
	& = &
	\inf_{\u\in\C:\reg^\L_\p(\pred(\u))\geq \reg^\L_\p(\pred(\u^*)) +\epsilon} 
		\p^\top \bpsi(\u) 
	\,.
\end{eqnarray}
Now, we claim that the mapping $\u \mapsto \reg^\L_\p(\pred(\u))$ is continuous at $\u = \u^*$. To see this, suppose the sequence $\u_m$ converges to $\u^*$. Then it is easy to see that $\p^{\u_m}$ converges to $\p^{\u^*} = \p$, and therefore for each $t\in[k]$, $(\p^{\u_m})^\top \bell_t$ converges to $\p^\top \bell_t$. Since by definition of $\pred$ we have that for all $m$, $\pred(\u_m) \in \argmin_{t} (\p^{\u_m})^\top \bell_t$, this implies that for all large enough $m$, $\pred(\u_m) \in \argmin_{t} \p^\top \bell_t$. Thus for all large enough $m$, $\reg^\L_\p(\pred(\u_m)) = 0$; i.e.\ the sequence $\reg^\L_\p(\pred(\u_m))$ converges to $\reg^\L_\p(\pred(\u^*))$, yielding continuity at $\u^*$. In particular, this implies $\exists \delta>0$ such that 
\[
\|\u - \u^*\| < \delta \implies \reg^\L_\p(\pred(\u)) - \reg^\L_\p(\pred(\u^*)) < \epsilon
	\,.
\]
This gives
\begin{eqnarray}
\inf_{\u\in\C: \reg^\L_\p(\pred(\u)) \geq \reg^\L_\p(\pred(\u^*)) +\epsilon} 
	\p^\top \bpsi(\u) 
  	& \geq & 
	\inf_{\u\in\C: \|\u - \u^*\| \geq \delta} \p^\top \bpsi(\u) 
\\
	& > &
	\inf_{\u\in\C} \p^\top \bpsi(\u)
	\,,
 \end{eqnarray}
where the last inequality holds since $\p^\top \bpsi(\u)$ is a strictly convex function of $\u$ and $\u^*$ is its unique minimizer. The above sequence of inequalities give us that
\begin{eqnarray}
\inf_{\u \in\C:\pred(\u)\notin \argmin_t{\p^\top\bell_t} } \p^\top \bpsi(\u) 
	& > &
	\inf_{\u\in\C} \p^\top \bpsi(\u)
	\,.
 \end{eqnarray}
Since this holds for all $\p\in\Delta_n$, we have that $\bpsi$ is $\L$-calibrated.
\end{proof}


\subsection{Proof of \Thm{thm:ccdim-lower-bound}}

The proof will require the lemma below, which relates the feasible subspace dimensions of different trigger probability sets at points in their intersection; we will also make critical use of the notion of $\epsilon$-subdifferentials of convex functions \citep{Bertsekas+03}, the main properties of which are also recalled below.

\begin{lem}
\label{lem:cone-dims-equal}
Let $\ell:[n]\times[k]\>\R_+^n$. Let $\p\in\relint(\Delta_n)$. Then for any $t_1,t_2\in\arg\min_{t'} \p^\top\bell_{t'}$ (i.e.\ such that $\p \in \Q^\L_{t_1}\cap \Q^\L_{t_2}$),
\[
\mu_{\Q^\L_{t_1}}(\p) = \mu_{\Q^\L_{t_2}}(\p)
    \,.
\]
\end{lem}
\begin{proof} (Proof of Lemma \ref{lem:cone-dims-equal})
\\[6pt]
Let $t_1,t_2\in\arg\min_{t'} \p^\top\bell_{t'}$ (i.e.\ $\p\in\Q^\L_{t_1} \cap \Q^\L_{t_2}$).
Now
\[
\Q^\L_{t_1}
	=
	\big\{ \q\in\R^n: -\q\leq \0, \e_n^\top\q = 1, (\bell_{t_1}-\bell_t)^\top \q \leq 0 ~\forall t\in[k] \big\}
	\,,
\]
where $\e_n$ denotes the $n\times 1$ all ones vector.
Moreover, we have $-\p<\0$, and $(\bell_{t_1} - \bell_t)^\top \p = 0$ iff $\p\in \Q^\L_t$. 
Let $\big\{ t\in[k]: \p\in \Q^\L_t \big\} = \big\{ \tilde{t}_1,\ldots, \tilde{t}_r \big\}$ for some $r\in[k]$. Then by \Lem{lem:feasible-subspace-dimension}, we have
\[
\mu_{Q^\L_{t_1}} = \nullity(\A_1) 
	\,,
\]
where $\A_1 \in \R^{(r+1)\times n}$ is a matrix containing $r$ rows of the form $(\bell_{t_1} - \bell_{\tilde{t}_j})^\top, j\in[r]$ and the all ones row.
Similarly, we get 
\[
\mu_{Q^\L_{t_2}} = \nullity(\A_2) 
	\,,
\]
where $\A_2 \in \R^{(r+1)\times n}$ is a matrix containing $r$ rows of the form $(\bell_{t_2} - \bell_{\tilde{t}_j})^\top, j\in[r]$ and the all ones row. It can be seen that the subspaces spanned by the first $r$ rows of $\A_1$ and $\A_2$ are both equal to the subspace parallel to the affine space containing $\bell_{\tilde{t}_1},\ldots,\bell_{\tilde{t}_r}$. Thus both $\A_1$ and $\A_2$ have the same row space and hence the same null space and nullity, and therefore $\mu_{\Q^\L_{t_1}}(\p)=\mu_{\Q^\L_{t_2}}(\p)$.
\end{proof}

\noindent
\textbf{$\epsilon$-Subdifferentials of a Convex Function.}
For any $\epsilon > 0$, the $\epsilon$-subdifferential of a convex function $\phi:\R^d\>\bbar{\R}$ at a point $\u_0\in\R^d$ is defined as follows \citep{Bertsekas+03}:
\[
\partial_\epsilon \phi(\u_0) 
	= 
	\big\{
		\w\in\R^d: \phi(\u) - \phi(\u_0) \geq \w^\top(\u-\u_0) - \epsilon ~~\forall \u\in\R^d
	\big\}
	\,.
\]
We recall some important properties of $\epsilon$-subdifferentials below:
\begin{itemize}
\item 
\(
\0 \in \partial_\epsilon \phi(\u_0) 
	~\Longleftrightarrow~ 
	\displaystyle{ \phi(\u_0) \leq \inf_{\u\in\R^d} \phi(\u) + \epsilon }
	\,.
\)
\item 
For any $\lambda>0$,
\(
\partial_\epsilon (\lambda \phi(\u_0)) 
	~ = ~
	\lambda \, \partial_{(\epsilon/\lambda)} \phi(\u_0)
	\,.
\)
\item 
If $\phi = \phi_1+\ldots+\phi_n$ for some convex functions $\phi_i:\R^d\>\bbar{\R}$,
then
\[
\partial_\epsilon \phi(\u_0) 
	~ \subseteq ~    
	\partial_\epsilon \phi_1(\u_0) + \ldots + \partial_\epsilon \phi_n(\u_0) 
	~ \subseteq ~
	\partial_{n\epsilon} \phi(\u_0)
	\,.
\] 
\item 
\(
\epsilon_1 \leq \epsilon_2 
	~\implies~
	 \partial_{\epsilon_1} \phi(\u_0) \subseteq \partial_{\epsilon_2} \phi(\u_0)
	 \,.
\)
\end{itemize}
We are now ready to give the proof of the lower bound on the CC dimension.
\\

\begin{proof} (Proof of Theorem \ref{thm:ccdim-lower-bound})
\\[6pt]
Let $d\in\Z_+$ be such that there exists a convex set $\C\subseteq\R^d$ and surrogate loss $\bpsi:\C\>\R_+^n$  such that $\bpsi$ is $\L$-calibrated. We will show that $d \geq \|\p\|_0 - \mu_{\Q^\L_t}(\p) - 1$.
We consider two cases:
\begin{enumerate}
\item[Case 1:] $\p\in\relint(\Delta_n)$.
\\[6pt]
In this case $\|\p\|_0 = n$. We will show that there exist $\H\subseteq \Delta_n$ and $t_0 \in[k]$ satisfying the following three conditions:
	\begin{enumerate}
	\item
	$\p\in\H \,;$
	\item
	$\mu_{\H}(\p) = n - d - 1 \,;$ and
	\item
	$\H \subseteq \Q^\L_{t_0}$.
	\end{enumerate}
Clearly, conditions (a) and (c) above imply $\p\in\Q^\L_{t_0}$. Conditions (b) and (c) will then give 
\[
\mu_{\Q^\L_{t_0}}(\p) 
	~ \geq ~
	\mu_{\H}(\p)
	~ = ~
	n - d - 1
	\,.
\]
Further, by \Lem{lem:cone-dims-equal}, we will then have that 
\[
\mu_{\Q^\L_{t}}(\p) 
	~ = ~
	\mu_{\Q^\L_{t_0}}(\p) 
	~ \geq ~
	n - d - 1
	\,,
\]
thus proving the claim.
\\[6pt]
We now show how to  construct $\H$ and $t_0$ satisfying the above conditions.
Let $\{\u_m\}$ be a sequence in $\C$ such that 
\[
\p^\top \bpsi(\u_m) \, \longrightarrow \, \inf_{\z\in\S_\psi} \p^\top \z = \inf_{\u\in\C} \p^\top \bpsi(\u)
	\,.
\]
Let 
\[
\epsilon_m = \p^\top \bpsi(\u_m) - \inf_{\u\in\C} \p^\top \bpsi(\u)
	\,.
\]
Then clearly $\epsilon_m \, \> \, 0$.
Now, for each $m$, we have
\begin{eqnarray*}
\0  
	& \in &
	\partial_{\epsilon_m} (\p^\top \bpsi(\u_m)) 
\\
	& \subseteq &
	\partial_{\epsilon_m}(p_1\psi_1(\u_m)) + \ldots + \partial_{\epsilon_m}(p_n\psi_n(\u_m)) 
\\
	& = &
	p_1\, \partial_{(\epsilon_m/p_1)}(\psi_1(\u_m)) + \ldots + p_n \, \partial_{(\epsilon_m/p_n)}(\psi_n(\u_m)) 
	\,.
\end{eqnarray*}
Therefore $\exists \w^m_y \in \partial_{(\epsilon_m/p_y)}(\psi_y(\u_m)) ~\forall y\in[n]$ such that
\[
\sum_{y=1}^n p_y \w^m_y = \0
	\,.
\] 
Let 
\[
\A^m = \big[ \w^m_1 \ldots \w^m_n \big] \in \R^{d\times n}
	\,,
\]
and define 
\begin{eqnarray*}
\H_m 
	& = & \big\{ \q\in\Delta_n: \A^m \q = \0 \big\} 
\\
	& = & \big\{ \q\in\R^n : \A^m\q = \0\,, \e_n^\top\q = 1\,, - \q \leq \0 \big\}
	\,,
\end{eqnarray*}
where $\e_n$ denotes the $n\times 1$ all ones vector.
Clearly, $\p\in\H_m$ and $-\p<\0$; therefore by \Lem{lem:feasible-subspace-dimension}, we have
\[
\mu_{\H_m}(\p) 
	~ = ~ 
	\nullity\left( \left[ \begin{array}{c} \A^m \\ \e_n^\top \end{array} \right] \right)
	~\geq~
	n - (d+1)
	\,.
\] 
This means that there exist $(n-d-1)$ orthonormal vectors $\v^m_1,\ldots,\v^m_{n-d-1} \in\R^n$ whose span $\V^m = \span(\{\v^m_1,\ldots,\v^m_{n-d-1}\})$ is contained in $\F_{\H_m}(\p) \cap (-\F_{\H_m}(\p))$.
It can be verified that this in turn implies 
\[
\{ \p+\v: \v\in\V^m \} \cap \Delta_n ~ \subseteq ~ \H_m
	\,.
\]
Now, $\{(\v^m_1,\ldots,\v^m_{n-d-1})\}$ is a bounded sequence in $(\R^n)^{n-d-1}$ and must therefore have a convergent subsequence, say with indices $r_1,r_2,\ldots$, converging to some limit point $(\v_1,\ldots,\v_{n-d-1})\in(\R^n)^{n-d-1}$. It can be verified that $\v_1,\ldots,\v_{n-d-1}$ must also form an orthonormal set of vectors. Let $\V = \span(\{\v_1,\ldots,\v_{n-d-1}\})$, and define
\[
\H ~ = ~ \{\p+\v: \v\in\V \} \cap \Delta_n
	\,.
\]
Clearly $\p\in\H$, and moreover, since $\p\in\relint(\Delta_n)$, we have $\mu_\H(\p) = n - d- 1$, thus satisfying conditions (a) and (b) above. 
\\[6pt]

For condition (c), we will show that $\H\subseteq \cN^\psi(\{\z_{m}\})$, where $\z_{m} = \bpsi(\u_{r_m})$; the claim will then follow from \Thm{thm:condition-necessary-sequence}. Consider any point $\q\in\H$. By construction of $\H$, we must have that $\q$ is the limit point of some convergent sequence $\{\q_m\}$ in $\Delta_n$ satisfying $\q_m \in \H_{r_m} ~\forall m$, i.e.\ $\A^{r_m}\q_m = \0 ~\forall m$. Therefore for each $m$, we have
\begin{eqnarray*}
\0 	~ = ~ 
	\sum_{y=1}^n q_{m,y} \w^{r_m}_y 
	&\in& 
	\sum_{y=1}^n q_{m,y} \, \partial_{(\epsilon_{r_m}/p_y)}(\psi_y(\u_{r_m})) 
\\
	&=& 
	\sum_{y=1}^n \partial_{(\epsilon_{r_m} q_{m,y}/p_y)}(q^m_y \, \psi_y(\u_{r_m})) 
\\
	&\subseteq& 
	\sum_{y=1}^n \partial_{(\epsilon_{r_m}/p_{\min})}(q_{m,y} \, \psi_y(\u_{r_m})) 
\\
	&\subseteq& 
	\partial_{(n\epsilon_{r_m}/p_{\min})}({\q_m}^\top \bpsi(\u_{r_m})) 
	\,,
\end{eqnarray*}
where $p_{\min} = \min_{y\in[n]} p_y > 0$ (since $\p\in\relint(\Delta_n)$).
This gives for each $m$:
\[
{\q_m}^\top \z_m
	~ = ~ 
	{\q_m}^\top \bpsi(\u_{r_m}) 
	~\leq~ 
	\inf_{\u\in\R^d}  {\q_m}^\top \bpsi(\u) + \frac{n\epsilon_{r_m}}{p_{\min}} 
	~ = ~
	\inf_{\z\in\S_\psi}  {\q_m}^\top \z + \frac{n\epsilon_{r_m}}{p_{\min}} 
	\,.
\]
Taking limits as $m\,\>\,\infty$, we thus get
\begin{equation}
 \label{eqn:CCdim-LB-proof-1} 
 \lim_{m\>\infty}{\q_m}^\top \z_m \leq \lim_{m\>\infty} \inf_{\z\in\S_\psi}  {\q_m}^\top \z \,.
\end{equation}
Now, since $\{\z_m\}$ is bounded, we have
\begin{equation}
\label{eqn:CCdim-LB-proof-2} 
 \lim_{m\>\infty}{\q_m}^\top \z_m = \lim_{m\>\infty}(\q_m-\q)^\top \z_m + \lim_{m\>\infty}{\q}^\top \z_m = \lim_{m\>\infty}{\q}^\top \z_m \,.
\end{equation}
Moreover, since the mapping $\p \mapsto \inf_{\z\in\S_\psi} \p^\top \z$ is continuous over its domain $\Delta_n$ (see \Lem{lem:continuous}), we have
\begin{equation}
\label{eqn:CCdim-LB-proof-3}
 \lim_{m\>\infty} \inf_{\z\in\S_\psi}  {\q_m}^\top \z = \inf_{\z\in\S_\psi}  {\q}^\top \z \,.
\end{equation}
Putting together Equations (\ref{eqn:CCdim-LB-proof-1}), (\ref{eqn:CCdim-LB-proof-2}) and (\ref{eqn:CCdim-LB-proof-3}), we therefore get
\begin{equation}
 \lim_{m\>\infty}{\q}^\top \z_m = \inf_{\z\in\S_\psi}  {\q}^\top \z 
 \,.
\end{equation}
Thus $\q\in\cN^\psi(\{\z_m\})$. Since $\q$ was an arbitrary point in $\H$, this gives $\H \subseteq \cN^\psi(\{\z_m\})$. The claim follows.

\item[Case 2:] $\p\notin\relint(\Delta_n)$.
\\[6pt] 
For each $\b\in\{0,1\}^n\setminus\{\0\}$, define 
\[
\cP^\b = \big\{ \q\in\Delta_n: q_y > 0 \Longleftrightarrow b_y = 1 \big\}
	\,.  
\]
Clearly, the set $\{\cP^\b:\b\in\{0,1\}^n\setminus\{\0\}\}$ forms a partition of $\Delta_n$. Moreover, for $\b=\e_n$ (the $n\times 1$ all ones vector), we have 
\[
\cP^{\e_n} 
	~ = ~ \big\{ \q\in\Delta_n: q_y > 0 ~\forall y\in[n] \big\}
	~ = ~ \relint(\Delta_n)
	\,.
\]
Therefore we have $\p\in\cP^\b$ for some $\b\in \{0,1\}^n\setminus\{\0, \e_n\}$, with $\|\p\|_0 = \|\b\|_0$.
Now, define $\bpsi^\b:\C\>\R_+^{\|\b\|_0}$, $\L^\b\in\R_+^{\|\b\|_0 \times k}$, and $\p^\b \in \Delta_{\|\b\|_0}$ as projections of $\bpsi$, $\L$ and $\p$ onto the $\|\b\|_0$ coordinates $y: b_y = 1$, so that $\bpsi^\b(\u)$ contains the elements of $\bpsi(\u)$ corresponding to coordinates $y:b_y = 1$, the columns $\bell^\b_t$ of $\L^\b$ contain the elements of the columns $\bell_t$ of $\L$ corresponding to the same coordinates $y:b_y=1$, and similarly, $\p^\b$ contains the strictly positive elements of $\p$. Since $\bpsi$ is $\L$-calibrated, and therefore in particular is calibrated w.r.t.\ $\L$ over $\{\q\in\Delta_n: q_y = 0 ~\forall y: b_y = 0\}$, we have that $\bpsi^\b$ is $\L^\b$-calibrated (over $\Delta_{\|\b\|_0}$). Moreover, by construction, we have $\p^\b \in \relint(\Delta_{\|\b\|_0})$. Therefore by Case 1 above, we have
\[
d ~ \geq ~ \|\b\|_0 - \mu_{\Q^{\L^\b}_t}(\p^\b) - 1
	\,.
\]
The claim follows since $\mu_{\Q^{\L^\b}_t}(\p^\b) \leq \mu_{\Q^\L_t}(\p)$.
\end{enumerate}
\end{proof}



\subsection{Proof of \Prop{prop:PD-rank}}

\begin{proof}
We will establish $\rank(\tilde{\L}^\pair) \geq \frac{r(r-1)}{2}$ by showing the existence of $\frac{r(r-1)}{2}$ linearly independent rows in $\tilde\L^\pair$; the claim will then follow by combining this with the previously stated upper bound $\rank(\tilde{\L}^\pair) \leq \frac{r(r-1)}{2}$.
\\[6pt]
Consider the $\frac{r(r-1)}{2}$ rows of $\tilde\L^\pair$ corresponding to graphs consisting of single directed edges $(i,j)$ with $i<j$. We claim these rows are linearly independent. To see this, suppose for the sake of contradiction that this is not the case. 
Then one of these rows, say the row corresponding to the graph with directed edge $(a,b)$ for some $a,b\in[r], a<b$, can be written as a linear combination of the other rows:
\begin{equation}
\label{eqn:PD-rank}
\tilde{\ell}^\pair((a,b),\sigma) 
	~ = ~ 
	\sum_{1\leq i<j \leq r\,,\,(i,j)\neq(a,b)} c_{ij} \, \tilde{\ell}^\pair((i,j),\sigma) \qquad \forall \sigma \in \Pi_r
	\,,
\end{equation}
for some coefficients $c_{ij}\in\R$.
Now consider two permutations $\sigma, \sigma' \in \Pi_r$ such that
\begin{eqnarray*}
\sigma(a) & = & \sigma'(b)
\\ 
\sigma(b) & = & \sigma'(a)
\\
\sigma(i)	& = & \sigma'(i) ~\forall i\neq a,b
	\,.
\end{eqnarray*}
Then applying \Eqn{eqn:PD-rank} to these two permutations gives
\[
\tilde{\ell}^\pair((a,b),\sigma) ~ = ~ \tilde{\ell}^\pair((a,b),\sigma')
	\,.
\]
However from the definition of $\tilde{\L}^\pair$ it is easy to verify that the columns corresponding to these two permutations have identical entries in all rows except for the row corresponding to the graph $(a,b)$, giving
\begin{eqnarray*}
\tilde{\ell}^\pair((i,j),\sigma) & = & \tilde{\ell}^\pair((i,j),\sigma') \hspace{2em} \forall i<j , (i,j)\neq (a,b) 
	\,;
\\
\tilde{\ell}^\pair((a,b),\sigma) & \neq & \tilde{\ell}^\pair((a,b),\sigma')
	\,.
\end{eqnarray*}
This yields a contradiction, and therefore we must have that the $\frac{r(r-1)}{2}$ rows above are linearly independent, giving
\[
\rank(\tilde{\L}^\pair) \geq \frac{r(r-1)}{2} \,.
\]
The claim follows.
\end{proof}

\subsection{Proof of \Prop{prop:MAP-rank-lower-bound}}

\begin{proof} 
From \Eqn{eqn:MAP-low-rank}, we have that $\L^\MAP \in \R^{(2^r-1) \times r!}$ can be written as
\[
\L^\MAP = \e_{(2^r-1)} \e_{r!}^\top - \A\B
	\,, 
\]
where $\e_{(2^r-1)}$ and $\e_{r!}$ denote the $(2^r-1)\times 1$ and $r!\times 1$ all ones vectors, respectively, and where $\A \in \R^{(2^r-1) \times \frac{r(r+1)}{2}}$ and $\B \in \R^{\frac{r(r+1)}{2} \times r!}$ are given by
\begin{eqnarray*}
A_{\y,(i,j)} 
	& = &
	\frac{1}{\|\y\|_1} y_{i} y_{j} 
	~~~~~~\forall \y\in\{0,1\}^r \setminus\{\0\}, ~ i,j\in[r]: i\leq j
	\,,
\\
B_{(i,j),\sigma} 
	& = &
	\frac{1}{\max (\sigma(i),\sigma(j))} 
	~~~~~~\forall i,j\in[r]: i\leq j, ~ \sigma \in\Pi_r
	\,.
\end{eqnarray*}
We will show that 
\begin{equation}
\rank(\A) 
	~ \geq ~
	\frac{r(r+1)}{2}-1 
\label{eqn:rank-A}
\end{equation}
and
\begin{equation}
\rank(\B) 
	~ \geq ~ 
	\frac{r(r-1)}{2} \,.
\label{eqn:rank-B}
\end{equation}
The result will then follow, since we will then have 
\begin{eqnarray*}
\rank(\L^\MAP) 
	& = & 
	\rank\big( \e_{(2^r-1)} \e_{r!}^\top - \A \B \big) 
\\
	& \geq & 
	\rank(\A \B)-1 
\\
 	& \geq & 
	\rank(\B)-2 
	\,,
	~~~~\text{since $\A$ is away from full (column) rank by at most $1$}
\\
	& \geq & 
	\frac{r(r-1)}{2}-2
	\,.
\end{eqnarray*}
To see why \Eqn{eqn:rank-A} is true, consider the $2^r$ vectors $\v^{\alpha}\in \R^{2^r}$ defined as 
\[
v^\alpha_\y = \prod_{i \in \alpha} y_i 
	~~~~ \forall \alpha\subseteq [r] \,, \y\in\{0,1\}^r
	\,. 
\]
It is easy to see that these vectors form a basis in $\R^{2^r}$. The columns of $\A$ can be obtained from the $\frac{r(r+1)}{2}$ vectors $\v^\alpha$ corresponding to subsets $\alpha\subseteq[r]$ of sizes 1 and 2, by removing the element corresponding to $\y=\0$ and dividing all other rows corresponding to $\y\in\{0,1\}^r\setminus\{\0\}$ by $\|\y\|_1$. This establishes the lower bound on $\rank(\A)$ in \Eqn{eqn:rank-A}.
\\[6pt]
To see why \Eqn{eqn:rank-B} is true, let us make the dependence of $\B$ on $r$ explicit by denoting $\B_r=\B$, and observe that $\B_r$ can be decomposed as
\[
\B_r=\begin{bmatrix} 
		\B_{r-1}   	& \mathbf D \\ 
		\mathbf C 	& \E  
\end{bmatrix} \,,
\]
where the sub-matrix $\B_{r-1}\in\R^{\frac{r(r-1)}{2} \times (r-1)!}$ is obtained by taking the $\frac{r(r-1)}{2}$ rows $(i,j)$ in $\B_r$ with $i\leq j < r$ and the $(r-1)!$ columns $\sigma$ in $\B_r$ with $\sigma(r)=r$:
\vspace{2pt}
\begin{center}
\begin{tabular}{c|cc} 
			&  $\Upsilon=\{\sigma\in\Pi_r: \sigma(r)=r\}$	& $\Omega=\{\sigma\in\Pi_r: \sigma(r)\neq r\}$	\\[6pt] 
\hline	
$\Gamma=\{(i,j)\in[r]\times[r]: i\leq j<r\}$	&  $\B_{r-1}$   			& $\mathbf D$ \\[6pt]
$\Lambda=\{(i,j)\in[r]\times[r]: i\leq j=r\}$	&  $\mathbf C$ 				& $\E$  
\end{tabular}
\end{center}
\vspace{2pt}
Consider the matrix $\mathbf C \in \R^{r\times (r-1)!}$. 
Each entry in this matrix has the form $\frac{1}{\max (\sigma(i),\sigma(j))}$ with $i\leq j=r$ and $\sigma(r)=r$. Thus all entries in $\mathbf C$ are equal to $\frac{1}{r}$ and $\rank(\mathbf C) = 1$.
\\[6pt]
Next, consider the matrix $\E \in \R^{r\times (r!-(r-1)!)}$. 
We will show that there are $r-1$ linearly independent columns in $\E$. In particular, consider any permutations $\sigma^1,\sigma^2,\ldots,\sigma^{r-1}$ in the set $\Omega$ such that $\sigma^j(j)=r$ and $\sigma^j(r)=r-1$ (such permutations clearly exist). The sub-matrix of $\E$ corresponding to these columns is given by
\vspace{2pt}
\begin{center}
 \begin{tabular}{c|cccc}
  & $\sigma^1(1)=r$ & $\sigma^2(2)=r$ & $\ldots$ &  $\sigma^{r-1}(r-1)=r$ \\ 
  & $\sigma^1(r)=r-1$ & $\sigma^2(r)=r-1$ & $\ldots$ &  $\sigma^{r-1}(r)=r-1$ \\ 
\hline
$(1,r)$ & $\frac{1}{r}$ & $\frac{1}{r-1}$ & $\ldots$ & $\frac{1}{r-1}$ \\ 
$(2,r)$ & $\frac{1}{r-1}$ & $\frac{1}{r}$ & $\ldots$ & $\frac{1}{r-1}$ \\ 
$\vdots$ & $\vdots$		& $\vdots$		&	$\ddots$	& $\vdots$	     \\ 
$(r-1,r)$ & $\frac{1}{r-1}$ & $\frac{1}{r-1}$ &	\ldots	& $\frac{1}{r}$ \\ 
$(r,r)$ & $\frac{1}{r-1}$ & $\frac{1}{r-1}$ &	\ldots	& $\frac{1}{r-1}$ 
 \end{tabular}
\end{center}
\vspace{2pt}
Thus excluding the last row of $\E$, one gets a square $(r-1)\times(r-1)$ matrix with diagonal entries equal to $\frac{1}{r}$ and off-diagonal entries equal to $\frac{1}{r-1}$. The last row of $\E$ has all entries equal to $\frac{1}{r-1}$. Clearly, this gives $\rank(\E)=r-1$. Moreover, the span of the $r-1$ column vectors of $\E$ does not intersect with column space of $\mathbf C$ non-trivially, since it does not contain the all ones vector. 
This implies that the $r-1$ columns of $\B_r$ corresponding to the permutations $\sigma^1,\sigma^2,\ldots,\sigma^{r-1}\in\Omega$ (which yield the linearly independent columns of $\E$), together with the columns of $\B_r$ corresponding to permutations $\sigma\in\Upsilon$ that yield linearly independent columns of $\B_{r-1}$, are all linearly independent. Therefore, we have 
\[
\rank(\B_r) ~\geq~ \rank(\B_{r-1})+r-1 
	\,.
\]
Trivially, $\rank(\B_1)\geq 0$. Expanding the above recursion therefore gives 
\[
\rank(\B_r) ~\geq~ \frac{r(r-1)}{2} 
	\,. 
\]
This establishes the lower bound on $\rank(\B)$ in \Eqn{eqn:rank-B}.
\end{proof}


\acks{HGR is supported by a TCS PhD Fellowship. SA thanks the Department of Science and Technology (DST) of the Government of India for a Ramanujan Fellowship, and the Indo-US Science and Technology Forum (IUSSTF) for their support.}


\begin{appendix}

\section{Proofs of \Lem{lem:pred-pred'} and \Thm{thm:class-calibration-consistency}}
\label{app:proofs}

\subsection{Proof of \Lem{lem:pred-pred'}}

\begin{proof}
Let $\bpsi:\C\>\R_+^n$. 
We will show that $\exists~ \pred:\C\>[k]$ satisfying the condition in \Def{def:calibration} if and only if $\exists~ \pred':\S_\psi\>[k]$ satisfying the stated condition.
\\[6pt]
\emph{(`if' direction)}
First, suppose $\exists~ \pred':\S_\psi\>[k]$ such that 
\[
\forall \p\in\cP:
    ~~~~\inf_{\z\in\S_\psi:\pred'(\z)\notin\argmin_{t} \p^\top \bell_t} \p^\top \z
    ~ > ~
    \inf_{\z\in\S_\psi} \p^\top \z
    \,.
\]
Define $\pred:\C\>[k]$ as follows:
\[
\pred(\u) = \pred'(\bpsi(\u)) ~~~~\forall \u\in\C
	\,.
\]
Then for all $\p\in\cP$, we have
\begin{eqnarray*}
    ~~~~\inf_{\u\in\C:\pred(\u)\notin\argmin_{t} \p^\top \bell_t} \p^\top \bpsi(\u)
    & = &
    \inf_{\z\in\cR_\psi:\pred'(\z)\notin\argmin_{t} \p^\top \bell_t} \p^\top \z
\\
    & \geq &
    \inf_{\z\in\S_\psi:\pred'(\z)\notin\argmin_{t} \p^\top \bell_t} \p^\top \z
\\
    & > &
    \inf_{\z\in\S_\psi} \p^\top \z
\\
	& = &
    \inf_{\u\in\C} \p^\top \bpsi(\u)
    \,.
\end{eqnarray*}
Thus $\bpsi$ is $(\L,\cP)$-calibrated.
\\[6pt]
\emph{(`only if' direction)}
Conversely, suppose $\bpsi$ is $(\L,\cP)$-calibrated, so that $\exists~ \pred:\C\>[k]$ such that 
\[
\forall \p\in\cP:
    ~~~~\inf_{\u\in\C:\pred(\u)\notin\argmin_{t} \p^\top \bell_t} \p^\top \bpsi(\u)
    ~ > ~
    \inf_{\u\in\C} \p^\top \bpsi(\u)
    \,.
\]
By Caratheodory's Theorem (e.g.\ see \cite{Bertsekas+03}), we have that every $\z\in\S_\psi$ can be expressed as a convex combination of at most $n+1$ points in $\cR_\psi$, i.e.\ for every $\z\in\S_\psi$, $\exists \balpha\in\Delta_{n+1}, \u_1,\ldots,\u_{n+1}\in\C$ such that $\z=\sum_{j=1}^{n+1} \alpha_j \bpsi(\u_j)$; w.l.o.g., we can assume $\alpha_1 \geq \frac{1}{n+1}$. For each $\z\in\S_\psi$, arbitrarily fix a unique such convex combination, i.e.\ fix $\balpha^\z\in\Delta_{n+1}, \u^\z_1,\ldots,\u^\z_{n+1}\in\C$ with $\alpha^\z_1 \geq \frac{1}{n+1}$ such that 
\[
\z=\sum_{j=1}^{n+1}\alpha_j^\z \bpsi(\u^\z_j) 
	\,.
\]
Now, define $\pred':\S_\psi\>[k]$ as follows:
\[
\pred'(\z)=\pred(\u^\z_1)
	~~~~\forall \z\in\S_\psi
	\,.
\]
Then for any $\p\in\cP$, we have
\begin{eqnarray*}
\inf_{\z\in\S_\psi:\pred'(\z)\notin\argmin_{t} \p^\top \bell_t} \p^\top \z 
	&=& 
	\inf_{\z\in\S_\psi:\pred(\u^\z_1)\notin\argmin_{t} \p^\top \bell_t} 
	\sum_{j=1}^{n+1}\alpha_j^\z \p^\top \bpsi(\u^\z_j)
\\
	&\geq& 
	\inf_{\balpha\in\Delta_{n+1},\u_1,\ldots,\u_{n+1} \in \C:\alpha_1\geq\frac{1}{n+1},\pred(\u_1)\notin\argmin_{t} \p^\top \bell_t} 
	\sum_{j=1}^{n+1} \alpha_j \p^\top \bpsi(\u_j) 
\\
	&\geq&
	\inf_{\balpha\in\Delta_{n+1}:\alpha_1\geq\frac{1}{n+1}} \sum_{j=1}^{n+1} \inf_{\u_j \in \C:\pred(\u_1)\notin\argmin_{t} \p^\top \bell_t} \alpha_j \p^\top \bpsi(\u_j) 
\\
	&\geq& 
	\inf_{\alpha_1\in[\frac{1}{n+1},1]} \alpha_1 \inf_{\u\in\C:\pred(\u)\notin\argmin_{t} \p^\top \bell_t} \p^\top \bpsi(\u) + (1-\alpha_1) \sum_{j=2}^{n+1} \inf_{\u \in\C} \p^\top \bpsi(\u)
\\
	&>& 
	\inf_{\u\in\C} \p^\top \bpsi(\u) 
\\
	& = &
	\inf_{\z\in\S_\psi} \p^\top \z 
	\,.
\end{eqnarray*}
Thus $\pred'$ satisfies the stated condition.
\end{proof}

\subsection{Proof of \Thm{thm:class-calibration-consistency}}


The proof is similar to that for the multiclass 0-1 loss given by \cite{TewariBa07}. We will make use of the following two lemmas; the first is a straightforward generalization of a similar lemma in \citep{TewariBa07}, and the second follows directly from \Lem{lem:pred-pred'}.

\begin{lem}
\label{lem:continuous}
The map $\p\mapsto \inf_{\z\in\S_\psi}\p^\top \z$ is continuous over $\Delta_n$.
\end{lem}

\begin{lem}
\label{lem:CCequiv-sequence}
Let $\L\in\R_+^{n\times k}$. A surrogate $\bpsi:\C\>\R^n$ is $\L$-calibrated if and only if there exists a function $\pred':\S_\psi\>[k]$ such that the following holds: for all $\p\in\Delta_n$ and all sequences $\{\z_m\}$ in $\S_\psi$ such that
$\lim_{m\>\infty} \p^\top \z_m =  \inf_{\z\in\S_\psi} \p^\top \z$, 
we have $\p^\top \bell_{\pred'(\z_m)} = \min_{t\in[k]} \p^\top \bell_t$ for all large enough $m$. 
\end{lem}

\begin{proof} (Proof of \Thm{thm:class-calibration-consistency}) \\[6pt]
Let $\bpsi:\C\>\R_+^n$.
\\[6pt]
\emph{(`only if' direction)}
First, suppose $\bpsi$ is $\L$-calibrated. Then by \Lem{lem:pred-pred'}, $\exists~ \pred':\S_\psi\>[k]$ such that 
\[
\forall \p\in\cP:
    ~~~~\inf_{\z\in\S_\psi:\pred'(\z)\notin\argmin_{t} \p^\top \bell_t} \p^\top \z
    ~ > ~
    \inf_{\z\in\S_\psi} \p^\top \z
    \,.
\]
Now, for each $\epsilon>0$, define 
\[
H(\epsilon)
	=
		\inf_{\p\in\Delta_n,\z\in\S_\psi: \p^\top \bell_{\pred'(\z)} - \min_{t\in[k]}\p^\top \bell_t \geq \epsilon} 
		\Big\{ 
			\p^\top \z - \inf_{\z\in \S_\psi} \p^\top \z 
		\Big\}
	\,.
\]
We claim that $H(\epsilon)>0 ~\forall\epsilon>0$.
Assume for the sake of contradiction that $\exists \epsilon > 0$ for which $H(\epsilon)=0$. Then there must exist a sequence $(\p_m,\z_m)$ in $\Delta_n \times \S_\psi$ such that
\begin{equation}
	{\p_m}^\top \bell_{\pred'(\z_m)} - \min_{t\in[k]}{\p_m}^\top \bell_t \geq \epsilon
	~~~\forall m
\label{eqn:proof-cal-cons-1}
\end{equation}
and
\begin{equation}
	{\p_m}^\top \z_m - \inf_{\z\in\S_\psi} {\p_m}^\top \z~\>~0 
	\,.
\label{eqn:proof-cal-cons-2}
\end{equation}
Since $\p_m$ come from a compact set, we can choose a convergent subsequence (which we still call $\{\p_m\}$), say with limit $\p$. Then by \Lem{lem:continuous}, we have $\inf_{\z\in\S_\psi} {\p_m}^\top \z 	~ \longrightarrow ~ \inf_{\z\in\S_\psi} \p^\top \z$, and therefore by \Eqn{eqn:proof-cal-cons-2}, we get 
\[
{\p_m}^\top \z_m ~ \longrightarrow ~ \inf_{\z\in\S_\psi} \p^\top \z
	\,.
\]
Now we show that $\z_m$ is a sequence such that $\p^\top \z_m ~ \longrightarrow ~ \inf_{\z\in\S_\psi} \p^\top \z$. Without loss of generality, we assume that the first $a$ coordinates of $\p$ are non-zero and the rest are zero. Hence the first $a$ coordinates of $\z_m$ are bounded for sufficiently large $m$, and we have
\[
\limsup_m \p^\top \z_m 
	= \limsup_m \sum_{y=1}^a p_{m,y} z_{m,y} 
	\leq \lim_{m\>\infty} {\p_m}^\top \z_m 
	= \inf_{\z\in\S_\psi} \p^\top \z
	\,.
\]
By \Lem{lem:CCequiv-sequence}, we therefore have $\p^\top \bell_{\pred'(\z_m)} = \min_{t\in[k]} \p^\top \bell_t$ for all large enough $m$, which contradicts \Eqn{eqn:proof-cal-cons-1} as $\p_m$ converges to $\p$.
Thus we must have $H(\epsilon) > 0 ~\forall \epsilon > 0$; the rest of the proof then follows from \citep{Zhang04a}.
\\[6pt]
\emph{(`if' direction)}
\\[6pt]
Conversely, suppose $\bpsi$ is not $\L$-calibrated. Consider any $\pred:\C\>[k]$. Then $\exists \p\in\Delta_n$ such that
\[
\inf_{\u\in\C:\pred(\u)\notin\argmin_{t} \p^\top \bell_t} \p^\top \bpsi(\u) 
	~ = ~ 
	\inf_{\u\in\C} \p^\top \bpsi(\u)
	\,.
\]
In particular, this means there exists a sequence of points $\{\u_m\}$ in $\C$ such that 
\[
\pred(\u_m)\notin\argmin_{t} \p^\top \bell_t 
	~~~\forall m 
\]
and 
\[
\p^\top\bpsi(\u_m) 
	~ \longrightarrow ~ 
	\inf_{\u\in\C} \p^\top \bpsi(\u)
	\,.
\]
Now consider a data distribution $D=D_X \times D_{Y|X}$ on $\X\times[n]$, with $D_X$ being a point mass at some $x\in\X$ and $D_{Y|X=x}=\p$. Let $\f_m:\X\>\C$ be any sequence of functions satisfying $\f_m(x) = \u_m ~\forall m$.
Then we have 
\[
\er^\psi_D[\f_m] =\p^\top\bpsi(\u_m)
	\,;
	~~~
	\er_D^{\psi, *} = \inf_{\u\in\C} \p^\top \bpsi(\u)
\]
and 
\[
\er^\L_D[\pred\circ\f_m] = \p^\top \bell_{\pred(\u_m)}
	\,;
	~~~
	\er_D^{\L, *} = \min_t \p^\top \bell_t
	\,.
\]
This gives
\[
\er^\psi_D[\f_m] ~ \longrightarrow ~ \er_D^{\psi, *}
\]
but 
\[
\er^\L_D[\pred\circ \f_m] ~ \not\longrightarrow ~ \er_D^{\L, *}
	\,.
\]
This completes the proof.
\end{proof}

\section{Calculation of Trigger Probability Sets for \Fig{fig:trigger-prob}}
\label{app:calculations-trigger-prob}

\begin{enumerate}
\item[(a)] 0-1 loss $\ell^\zo$ ($n=3$).
    \[
    \bell_1 = \left( \begin{array}{c} 0 \\ 1 \\ 1 \end{array} \right) \,;~~
    \bell_2 = \left( \begin{array}{c} 1 \\ 0 \\ 1 \end{array} \right) \,;~~
    \bell_3 = \left( \begin{array}{c} 1 \\ 1 \\ 0 \end{array} \right) \,.
    \]
    \begin{eqnarray*}
    \Q^\zo_1
        & = &
        \{ \p\in\Delta_3: \p^\top\bell_1 \leq  \p^\top\bell_2, ~ \p^\top\bell_1 \leq  \p^\top\bell_3 \}
    \\
        & = &
        \{ \p\in\Delta_3: p_2 + p_3 \leq p_1 + p_3, ~ p_2 + p_3 \leq p_1 + p_2 \}
    \\
        & = &
        \{ \p\in\Delta_3: p_2 \leq p_1, ~ p_3 \leq p_1 \}
    \\
        & = &
        \{ \p\in\Delta_3: p_1 \geq \max(p_2,p_3) \}
    \end{eqnarray*}
    By symmetry,
    \begin{eqnarray*}
    \Q^\zo_2
        & = &
        \{ \p\in\Delta_3: p_2 \geq \max(p_1,p_3) \}
        \hspace{5cm}
    \\
    \Q^\zo_3
        & = &
        \{ \p\in\Delta_3: p_3 \geq \max(p_1,p_2) \}
        \hspace{5cm}
    \end{eqnarray*}
\item[(b)] Ordinal regression loss $\ell^\ord$ ($n=3$).
    \[
    \bell_1 = \left( \begin{array}{c} 0 \\ 1 \\ 2 \end{array} \right) \,;~~
    \bell_2 = \left( \begin{array}{c} 1 \\ 0 \\ 1 \end{array} \right) \,;~~
    \bell_3 = \left( \begin{array}{c} 2 \\ 1 \\ 0 \end{array} \right) \,.
    \]
    \begin{eqnarray*}
    \Q^\ord_1
        & = &
        \{ \p\in\Delta_3: \p^\top\bell_1 \leq  \p^\top\bell_2, ~ \p^\top\bell_1 \leq  \p^\top\bell_3 \}
    \\
        & = &
        \{ \p\in\Delta_3: p_2 + 2p_3 \leq p_1 + p_3, ~ p_2 + 2p_3 \leq 2p_1 + p_2 \}
    \\
        & = &
        \{ \p\in\Delta_3: p_2 + p_3 \leq p_1, ~ p_3 \leq p_1 \}
    \\
        & = &
        \{ \p\in\Delta_3: 1 - p_1 \leq p_1 \}
    \\
        & = &
        \{ \p\in\Delta_3: p_1 \geq {\textstyle{\half}} \}
    \end{eqnarray*}
    By symmetry,
    \begin{eqnarray*}
    \Q^\ord_3
        & = &
        \{ \p\in\Delta_3: p_3 \geq {\textstyle{\half}} \}
        \hspace{5cm}
    \end{eqnarray*}
    Finally,
    \begin{eqnarray*}
    \Q^\ord_2
        & = &
        \{ \p\in\Delta_3: \p^\top\bell_2 \leq  \p^\top\bell_1, ~ \p^\top\bell_2 \leq  \p^\top\bell_3 \}
    \\
        & = &
        \{ \p\in\Delta_3: p_1 + p_3 \leq p_2 + 2p_3, ~ p_1 + p_3 \leq 2p_1 + p_2 \}
    \\
        & = &
        \{ \p\in\Delta_3: p_1 \leq p_2 + p_3, ~ p_3 \leq p_1 + p_2 \}
    \\
        & = &
        \{ \p\in\Delta_3: p_1 \leq 1 - p_1, ~ p_3 \leq 1 - p_3 \}
    \\
        & = &
        \{ \p\in\Delta_3: p_1 \leq {\textstyle{\half}}, ~ p_3 \leq {\textstyle{\half}} \}
    \\
    \end{eqnarray*}
\item[(c)] `Abstain' loss $\ell^\abstain$ ($n=3$).
    \[
    \bell_1 = \left( \begin{array}{c} 0 \\ 1 \\ 1 \end{array} \right) \,;~~
    \bell_2 = \left( \begin{array}{c} 1 \\ 0 \\ 1 \end{array} \right) \,;~~
    \bell_3 = \left( \begin{array}{c} 1 \\ 1 \\ 0 \end{array} \right) \,;~~
    \bell_4 = \left( \begin{array}{c} \half \\ \half \\ \half \end{array} \right) \,.
    \]
    \begin{eqnarray*}
    \Q^\abstain_1
        & = &
        \{ \p\in\Delta_3: \p^\top\bell_1 \leq  \p^\top\bell_2, ~ \p^\top\bell_1 \leq  \p^\top\bell_3, ~ \p^\top\bell_1 \leq  \p^\top\bell_4 \}
    \\
        & = &
        \{ \p\in\Delta_3: p_2 + p_3 \leq p_1 + p_3, ~ p_2 + p_3 \leq p_1 + p_2, ~ p_2 + p_3 \leq {\textstyle\half}(p_1 + p_2 + p_3) \}
    \\
        & = &
        \{ \p\in\Delta_3: p_2 \leq p_1, ~ p_3 \leq p_1, ~ p_2 + p_3 \leq {\textstyle\half} \}
    \\
        & = &
        \{ \p\in\Delta_3: p_1 \geq {\textstyle\half} \}
    \end{eqnarray*}
    By symmetry,
    \begin{eqnarray*}
    \Q^\abstain_2
        & = &
        \{ \p\in\Delta_3: p_2 \geq {\textstyle\half} \}
        \hspace{5cm}
    \\
    \Q^\abstain_3
        & = &
        \{ \p\in\Delta_3: p_3 \geq {\textstyle\half} \}
        \hspace{5cm}
    \end{eqnarray*}
    Finally,
    \begin{eqnarray*}
    \Q^\abstain_4
        & = &
        \{ \p\in\Delta_3: \p^\top\bell_4 \leq  \p^\top\bell_1, ~ \p^\top\bell_4 \leq  \p^\top\bell_2, ~ \p^\top\bell_4 \leq  \p^\top\bell_2 \}
    \\
        & = &
        \{ \p\in\Delta_3: {\textstyle{\half}}(p_1 + p_2 + p_3) \leq \min(p_2 + p_3, p_1 + p_3, p_1 + p_2) \}
    \\
        & = &
        \{ \p\in\Delta_3: {\textstyle{\half}} \leq 1 - \max(p_1,p_2,p_3) \}
    \\
        & = &
        \{ \p\in\Delta_3: \max(p_1,p_2,p_3) \leq {\textstyle{\half}} \}
    \end{eqnarray*}
\end{enumerate}

\section{Calculation of Positive Normal Sets for \Fig{fig:pos-normals}}
\label{app:calculations-pos-normals}

For $n=3$, the Crammer-Singer surrogate $\bpsi^\CS:\R^3\>\R_+^3$ is given by
\begin{eqnarray*}
\psi^\CS_1(\u) &=& \max(1+u_2 - u_1, 1+u_3 - u_1, 0) \\ 
\psi^\CS_2(\u) &=& \max(1+u_1 - u_2, 1+u_3 - u_2, 0) \\
\psi^\CS_3(\u) &=& \max(1+u_1 - u_3, 1+u_2 - u_3, 0) 
	~~~\forall \u\in\R^3
	\,.
\end{eqnarray*}
Clearly, $\bpsi^\CS$ is a convex function. 
Let $\u_1 = (1,0,0)^\top$, $\u_2 = (0,1,0)^\top$, $\u_3 = (0,0,1)^\top$, $\u_4 = (0,0,0)^\top$, and let 
\begin{eqnarray*}
\z_1 & = & \bpsi^\CS(\u_1) ~ = ~ (0,2,2)^\top \\
\z_2 & = & \bpsi^\CS(\u_2) ~ = ~ (2,0,2)^\top \\
\z_3 & = & \bpsi^\CS(\u_3) ~ = ~ (2,2,0)^\top \\
\z_4 & = & \bpsi^\CS(\u_4) ~ = ~ (1,1,1)^\top
	\,.
\end{eqnarray*}
We apply \Lem{lem:normal-computation} to compute the positive normal sets of $\bpsi^\CS$ at the 4 points $\z_1,\z_2,\z_3,\z_4$ above. In particular, to see that $\z_4$ satisfies the conditions of \Lem{lem:normal-computation}, note that by Danskin's Theorem \citep{Bertsekas+03}, we have that 
\[
\partial \psi^\CS_1(\u_4) =\conv\left( \begin{bmatrix}
                                          -1 \\ +1 \\ 0
                                         \end{bmatrix},
					 \begin{bmatrix}
                                          -1 \\ 0 \\ +1
                                         \end{bmatrix} \right)
	\,;
\]
\[
\partial \psi^\CS_2(\u_4) =\conv\left( \begin{bmatrix}
                                          +1 \\ -1 \\ 0
                                         \end{bmatrix},
					 \begin{bmatrix}
                                          0 \\ -1 \\ +1
                                         \end{bmatrix} \right)
	\,;
\]
\[
\partial \psi^\CS_3(\u_4) =\conv\left( \begin{bmatrix}
                                          +1 \\ 0 \\ -1
                                         \end{bmatrix},
					 \begin{bmatrix}
                                          0 \\ +1 \\ -1
                                         \end{bmatrix} \right)
	\,.
\]
We can therefore use \Lem{lem:normal-computation} to compute $\cN^\CS(\z_4)$. Here $s=6$, and 
\[
\A=\begin{bmatrix}
     -1 & -1 & 1 & 0 & 1 & 0\\
      1 & 0 & -1 & -1 & 0 & 1\\
      0 & 1 & 0 & 1 & -1 & -1\\
    \end{bmatrix} \,; ~~~~
\B=\begin{bmatrix}
     1 & 1 & 0 & 0 & 0 & 0\\
      0 & 0 & 1 & 1 & 0 & 0\\
      0 & 0 & 0 & 0 & 1 & 1\\
    \end{bmatrix}
	\,.
\]
By \Lem{lem:normal-computation} (and some algebra), this gives
\begin{eqnarray*}
\cN^\CS(\z_4) 
	& = & 
	\big\{\p \in \Delta_3: \p = (q_1+q_2, q_3+q_4, q_5+q_6) ~\text{for some}~ 
				\q\in\Delta_6, 
\\
	& &
	~~~~~~~~~~~~ 
	q_1 + q_2 = q_3 + q_5, ~ 
	q_3 + q_4 = q_1 + q_6, ~
	q_5 + q_6 = q_2 + q_4 \big\} 
\\
	& = & 
	\textstyle{ \big\{\p \in \Delta_3: p_1 \leq \half, p_2 \leq \half, p_3 \leq \half \big\} }
	\,.
\end{eqnarray*}
It is easy to see that $\z_1,\z_2, \z_3$ also satisfy the conditions of  \Lem{lem:normal-computation}; similar computations then yield
\begin{eqnarray*}
\cN^\CS(\z_1) & = & \textstyle{ \big\{ \p \in \Delta_3: p_1 \geq \half \big\} } \\
\cN^\CS(\z_2) & = & \textstyle{ \big\{ \p \in \Delta_3: p_2 \geq \half \big\} } \\
\cN^\CS(\z_3) & = & \textstyle{ \big\{ \p \in \Delta_3: p_3 \geq \half \big\} }
	\,.
\end{eqnarray*}

\end{appendix}


\bibliography{ramaswamy15a}

\end{document}